\documentclass{article}

\usepackage{arxiv2}

\usepackage[utf8]{inputenc} 
\usepackage[T1]{fontenc}    
\usepackage{url}            
\usepackage{booktabs}       
\usepackage{amsfonts}       
\usepackage{nicefrac}       
\usepackage{microtype}      

\usepackage{times}
\usepackage{url}

\usepackage{bbm}
\usepackage{booktabs}       
\usepackage{amsfonts}       
\usepackage{nicefrac}       
\usepackage{microtype}      

\usepackage{dsfont}

\usepackage{color}
\usepackage{diagbox}
\usepackage{array,multirow,graphicx}
\usepackage{caption}
\usepackage{mathtools}

\usepackage{amsthm,amsmath,amssymb,graphicx}%
\usepackage{enumitem}
\setlist{nosep}

\usepackage{thmtools}
\usepackage{thm-restate}
\usepackage{cleveref}

\usepackage[ruled]{algorithm2e}
\SetKwProg{Init}{init}{}{}
\usepackage{tabularx}

\usepackage{csquotes}

\usepackage{subcaption}

\usepackage[percent]{overpic} 

\usepackage{import}

\newtheorem{definition}{Definition}
\newtheorem{proposition}{Proposition}

\newtheorem{lemma}{Lemma}

\newcommand{\icol}[1]{
  \left(\begin{smallmatrix}#1\end{smallmatrix}\right)%
}

\def\defRecurse{P}
\newcommand{\Recurse}[1]{\defRecurse^{(#1)}}
\def\Op{\Recurse{d}}
\newcommand{\Opl}[1]{\Recurse{#1}}
\newcommand{\DotOp}[2]{\defRecurse^{(d)}(#1,#2)}
\newcommand{\DotOperand}[3]{\defRecurse^{(#3)}(#1,#2)}
\def\x{x}
\def\defxvec{\tilde{\x}}
\def\xv{\defxvec{}}
\def\xcol{\icol{\x\\1}}

\def\defVCBool{2k\log_2(8esd)}
\def\VCbool{VC^{Bool}_{k,d,s}}
\def\VCind{VC^{\text{NoData}}}
\def\VCvap{\text{VCDim}}
\newcommand{\mul}[1]{\mu^{#1}}
\newcommand{\taul}[1]{\tau^{#1}}
\newcommand{\sigl}[1]{\sigma^{#1}}
\def\mubar{\bar{\mu}}
\def\taubar{\bar{\tau}}
\def\sigbar{\bar{\sigma}}

\def\Sigbar{\bar{\Sigma}}
\def\SigbarM{\bar{\Sigma}_{-}}
\def\SigbarO{\bar{\Sigma}_{0}}
\def\SigbarP{\bar{\Sigma}_{+}}

\newcommand{\Sigbarl}[1]{{\Sigbar}^{#1}}

\newcommand{\SigbarOl}[1]{{\SigbarO}^{#1}}

\def\Archi{ArchI}
\def\Archii{ArchII}
\def\Archiii{ArchIII}
\def\Datai{DataI}
\def\Dataii{DataII}
\def\Dataiii{DataIII}

\def\defwidth{n}
\newcommand{\width}[1]{\defwidth_{#1}}
\def\widthbar{\bar{\defwidth}}

\def\defrank{r}
\newcommand{\rankl}[1]{\defrank_{#1}}
\def\rankbar{\bar{\defrank}}
\def\defrankbar{(\rankl{0},\rankl{1},\ldots,\rankl{d})}

\def\depth{d}
\def\Wmat{A}
\newcommand{\WeightMatrix}[1]{\Wmat^{(#1)}}
\newcommand{\W}[1]{\WeightMatrix{#1}}
\newcommand{\bs}{b}

\newcommand{\B}[1]{\bs^{(#1)}}

\def\relu{R}

\def\net{\mathcal{N}}
\newcommand{\netl}[1]{\net^{(#1)}}

\usepackage{soul,xcolor}
\setstcolor{red}

\newcommand{\norm}[1]{\lVert #1\rVert}

\def\bbR{\mathbb{R}}

\def\pdata{\mathcal{D}}

\begin{document}

\title{Deep Networks as Logical Circuits:  Generalization and Interpretation}
\author{
Christopher Snyder\\ 
Department of Biomedical Engineering\\
The University of Texas at Austin\\ 
\texttt{22csnyder@gmail.com}
\and \textbf{Sriram Vishwanath}\\Department of Electrical and Computer Engineering\\
The University of Texas at Austin\\ 
\texttt{sriram@austin.utexas.edu}}
\maketitle

\begin{abstract}
Not only are Deep Neural Networks (DNNs) black box models, but also we frequently conceptualize them as such. We lack good interpretations of the mechanisms linking inputs to outputs. Therefore, we find it difficult to analyze in human-meaningful terms (1) what the network learned and (2) whether the network learned. We present a hierarchical decomposition of the DNN discrete classification map into logical (AND/OR) combinations of intermediate (True/False) classifiers of the input. Those classifiers that can not be further decomposed, called atoms, are (interpretable) linear classifiers. Taken together, we obtain a logical circuit with linear classifier inputs that computes the same label as the DNN. This circuit does not structurally resemble the network architecture, and it may require many fewer parameters, depending on the configuration of weights. In these cases, we obtain simultaneously an interpretation and generalization bound (for the original DNN), connecting two fronts which have historically been investigated separately. Unlike compression techniques, our representation is \textit{exact}. We motivate the utility of this perspective by studying DNNs in simple, controlled settings, where we obtain superior generalization bounds despite using only combinatorial information (e.g. no margin information). We demonstrate how to "open the black box" on the MNIST dataset. We show that the learned, internal, logical computations correspond to semantically meaningful (unlabeled) categories that allow DNN descriptions in plain English. We improve the generalization of an already trained network by interpreting, diagnosing, and replacing components \textit{within} the logical circuit that is the DNN.
\end{abstract}

\section{Introduction}
Deep Neural Networks (DNNs) are among the most widely studied and applied models, in part because they are able to achieve state-of-the-art performance on a variety of tasks such as predicting protein folding, object recognition, playing chess. Each of these domains was previously the realm of many disparate, setting-specific, algorithms. The underlying paradigm of Deep Learning (DL) is, by contrast, relatively similar across these varied domains. This suggests that the advantages of DL may be relevant in a variety of future learning applications rather than being restricted to currently-known settings.


The philosophy of investigating deep learning has typically focused upon keeping experimental parameters as realistic as possible. A key advantage enabled by this realism is that the insights from each experiment are immediately transferable to settings of interest. However, this approach comes with an important disadvantage: Endpoints from realistic experiments can be extremely noisy and complicated functions of variables of interest, even for systems with simple underlying rules. Newton's laws are simple, but difficult to discover except in the most controlled of settings.

The goal of our study is to understand the relationship between generalization error and network size. We seek to clarify why DNN architectures that can potentially fit all possible training labels are able to generalize to unseen data. Specifically, we would like to understand why increasing the capacity of a DNN (through increasing the number of layers and parameters) is not always accompanied by an increase in  test error. 
To this end we study fully-connected, Gaussian-initialized, unregularized, binary classification DNNs trained with gradient descent to minimize cross-entropy loss on $2$-dimensional data \footnote{Though the generalization of DNNs has been attributed in part to SGD, dropout, batch normalization, weight sharing (e.g. CNNs), etc., none of these are strictly necessarily to exhibit the apparent paradox we describe.}. Even in such simple settings, generalization is not yet well-understood (as bounds can be quite large for deep networks), and our goal is to take an important step in that direction.

In adopting a minimalist study of this generalization phenomenon, the view taken in this paper is aligned with that expressed by Ali Rahimi in the NIPS2017 "Test of Time Award" talk: 
"This is how we build knowledge. We apply our tools on simple, easy to analyze setups; we learn; and, we work our way up in complexity\textellipsis Simple experiments --- simple theorems are the building blocks that help us understand more complicated systems."

---\cite{Rahimi2017}

Our contributions are:
\begin{enumerate}
    \item We give an intuitive, visual explanation for generalization using experiments on simple data. We show that prior knowledge about the training data can imply regularizing constraints on the image of gradient descent independently of the architecture. We observe this effect is most pronounced at the decision boundary.
    \item We represent exactly a DNN classification map as a logical circuit with many times fewer parameters, depending on the data complexity.
    \item We demonstrate that our logical transformation is useful both for interpretation and improvement of trained DNNs. On the MNIST dataset we translate a network "into plain English". We improve the test accuracy of an already trained DNN by debugging and replacing \textit{within the logical circuit} of the DNN a particular intermediate computation that had failed to generalize.
    \item We give a formal explanation for generalization of deep networks on simple data using classical VC bounds for learning Boolean formulae. 
    Our bound is favorable to state of the art bounds that use more information (e.g. margin). Our bounds are extremely robust to increasing depth.
\end{enumerate}
\section{Setting and Notation}

In this paper we study binary classification ReLU fully connected deep neural networks, $\net:\bbR^{\width{0}}\mapsto\bbR$, that assign input $x$, label $y\in\{False,True\}$ according to the value $[\net(x)\geq 0]$. This network has $d$ hidden layers, each of width $\width{l}$, indexed by  $l=1,\ldots,d$. We
reserve the index $0$[$d+1$] for the input[output] space, so that $\width{d+1}=1$. Our ReLU nonlinearities, $\relu(x)_i=\max\{0,x_i\}$, are applied coordinate-wise, interleaving the affine maps defined by weights $\W{l}\in\bbR^{\width{l+1}\times\width{l}}, \B{l}\in\bbR^{\width{l+1}}$. These layers compute recursively 
\begin{align*}
\netl{l+1}(x)&\triangleq\B{l+1}+\W{l}\relu( \netl{l}(x) ).
\end{align*}
Here, we include the non-layer indices $0$ and $d+1$ to address the input, $x=\netl{0}(x)$, and the output, $\net(x)=\netl{d+1}(x)$, respectively.

For a particular input, $x$, each neuron occupies a binary "state" according to the sign of its activation. The set of inputs for which the activation of a given neuron is identically $0$ comprises a "neuron state boundary"(NSB), of which we consider the decision boundary to be a special case by convention. We can either group these states by layer or all together to designate either the layer state, $\sigl{l}(x)\in\{0,1\}^{\width{l}}$, or the network state, $\sigbar(x)=(\sigl{1}(x),\ldots,\sigl{d}(x))$, respectively.



We consider our training set, $\{(x_1,y_1),\ldots,(x_m,y_m)\}$, to represent samples from some distribution, $\pdata$. We define generalization error of the network to be the difference between the fraction correctly classified in the finite training set and in the overall distribution. We define training as the process that assigns parameters the final value of unregularized gradient descent on cross-entropy loss.

\section{Insights From a Controlled Setting}
\label{sec:insights}
\subsection{Finding the Right Question}
Note that one is \textit{not} always guaranteed small generalization error. There are many settings where DNNs under-perform and have high generalization error. For our purposes, it suffices to recall that when the inputs and outputs $(X,Y)\sim\pdata$ are actually independent, e.g., $Y|X\sim Bernoulli(1/2)$, neural networks still obtain zero empirical risk, which implies the generalization error can be arbitrarily bad in the worst case \citep{Zhang2017}. 
From this, we can conclude that making either an explicit or implicit assumption about the dataset, the data, or both, is \textit{strictly necessary} and unavoidable. At the very least, one must make an assumption which rules out random labels with high probability. 

Notice that the only procedural distinction between a DNN that will generalize and one that will memorize is the dataset. Those network properties capable of distinguishing learning from memorizing, e.g., Lipschitz constant or margin, must therefore arise as secondary characteristics. They are functions of the dataset the network is trained on.

We want clean descriptions of DNN functions that generalize. By the above discussion, these are the DNNs that inherit some regularizing property from the training data through the gradient descent process. What sort of architecture agnostic language allows for succinct descriptions of trained DNNs exactly when we make some strong assumption about the training set?


\begin{figure*}
\renewcommand{\tabcolsep}{0pt}
\renewcommand{\arraystretch}{0.15}
\def\nsbwid{0.15\textwidth}
\def\nsbht{1.4cm}
  \centering
    \begin{tabular}[t]{ccc}
\begin{subfigure}{0.5\textwidth}
    \centering
    \includegraphics[height=4.2cm,keepaspectratio]{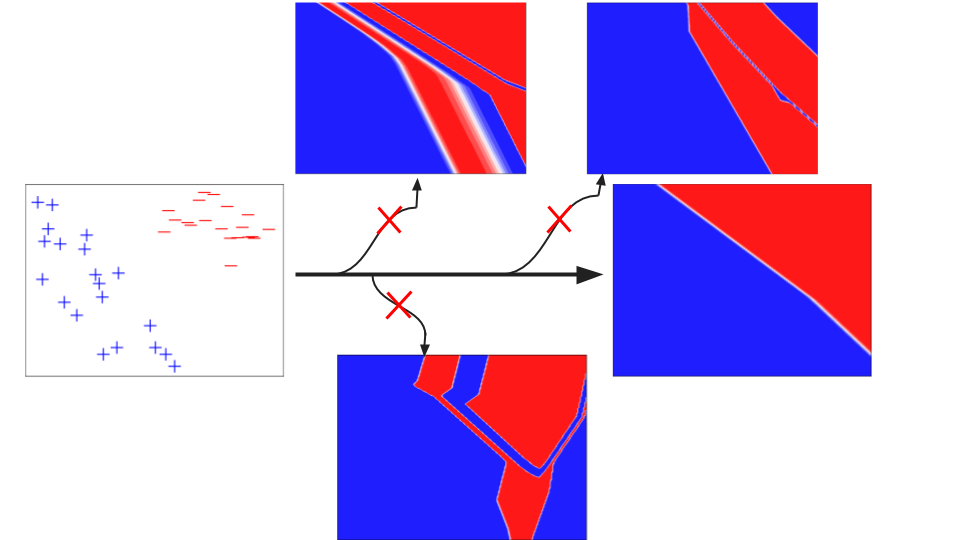}
    \caption{Regularizing  Effect of Linearly Separable Data}
    \label{fig:1-a}
\end{subfigure}
    &
        \begin{tabular}{c}
         \begin{subfigure}[b]{\nsbwid}
                \centering
                \includegraphics[height=\nsbht{}]{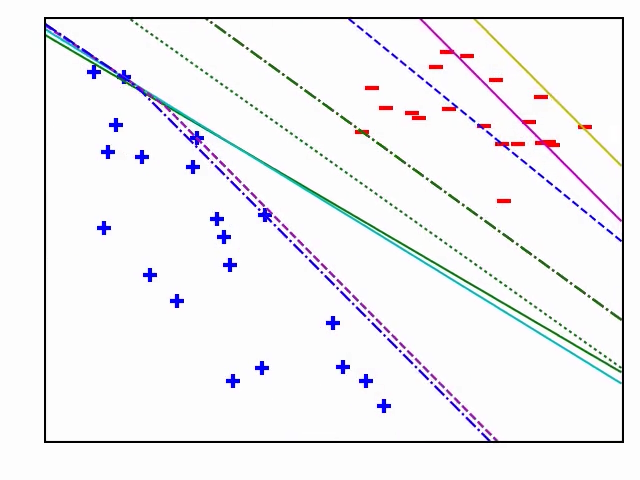}
         \end{subfigure}\\
         \begin{subfigure}[b]{\nsbwid}
                \centering
                \includegraphics[height=\nsbht{}]{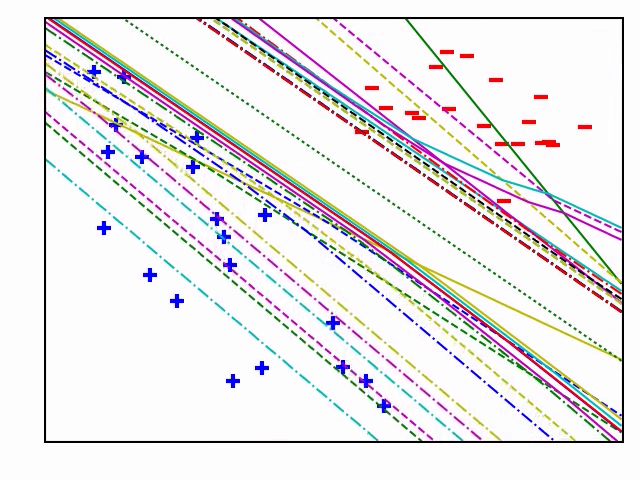}
         \end{subfigure}\\
         \begin{subfigure}[b]{\nsbwid}
                \centering
                \includegraphics[height=\nsbht{}]{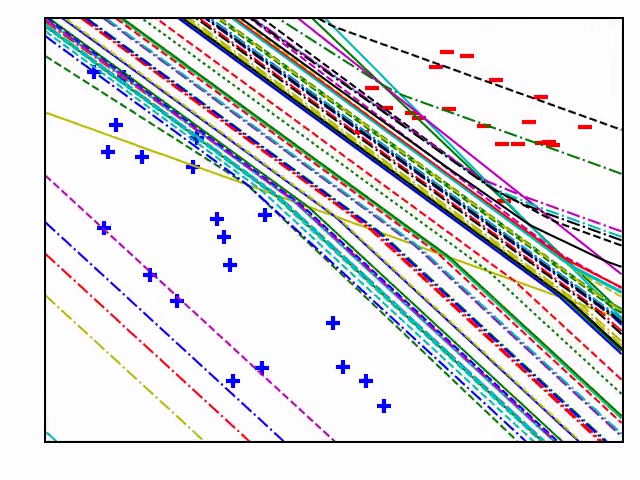}
                \caption{{\Datai}}
                \label{fig:1-b}
         \end{subfigure}
        \end{tabular}
&
        \begin{tabular}{c}
         \begin{subfigure}[b]{\nsbwid}
                \centering
                \includegraphics[height=\nsbht{}]{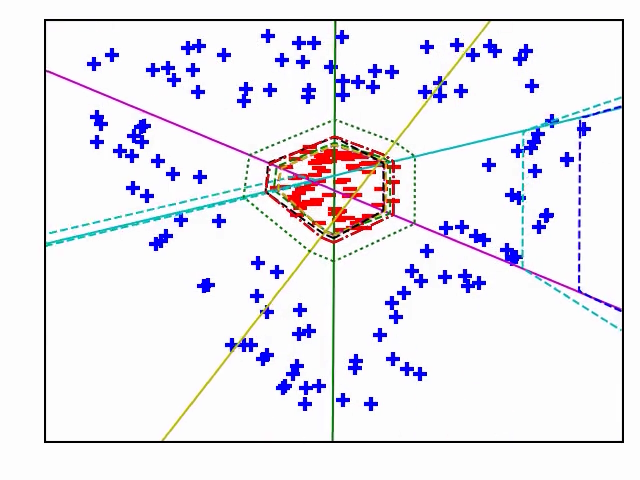}
         \end{subfigure}\\
         \begin{subfigure}{\nsbwid}
                \centering
                \includegraphics[height=\nsbht{}]{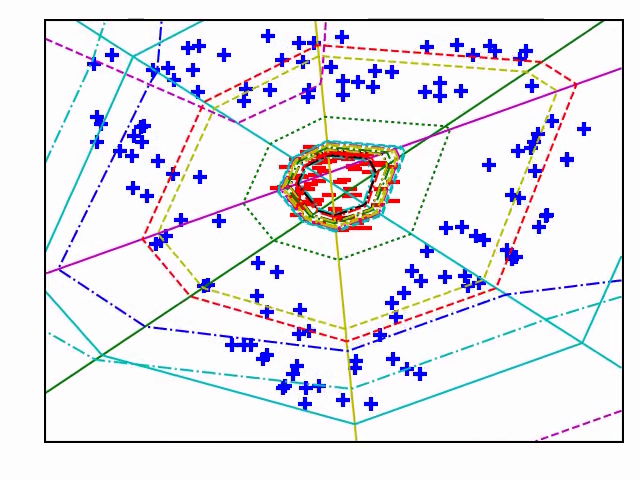}
         \end{subfigure}\\
         \begin{subfigure}{\nsbwid}
                \centering
                \includegraphics[height=\nsbht{}]{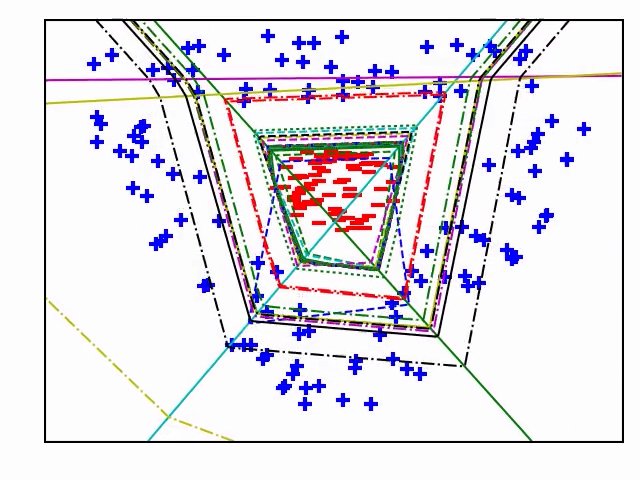}
                \caption{{\Dataii}}
                \label{fig:1-c}
         \end{subfigure}
        \end{tabular}\\
    \end{tabular}
  \caption{Structural organization of the decision boundary(DB) and NSBs (where each neuron changes from "on" to "off") of trained DNNs as the data (Fig. \ref{fig:1-a}) and architecture complexity varies (Figs.\ref{fig:1-b}, \ref{fig:1-c}). In Fig. \ref{fig:1-a}, \textit{if and only if} we include additional noisy training data to the linearly classifiable {\Datai} can we avoid learning an (essentially) linear classifier. Regularity is not tied to data fit but data structure: all $4$ {\Archiii} classifiers have $0$ training error and vanishing loss on the same original data.
  %
  In columns (\ref{fig:1-b},\ref{fig:1-c}) we plot all NSBs(different linestyle[color] distinguishes NSBs of neurons in different[the same] hidden layer) and the DB(dotted). We see that for fixed dataset, increasing the architecture size (moving down a column) does not qualitatively change the learned DB. Additional layers may add more NSBs, but these organize during training in redundant, parallel shells that do not make the DB more complex. 
  Only those NSBs that intersect the DB influence the DB and cause it to bend. Not only is the number of intersections between the DB and NSBs minimized, but also they separate from one another during training, as if by some (regularizing) "repulsive force", most readily apparent in row $2$ col \ref{fig:1-b} (and in the Supplemental animations), that repels the NSBs from the decision boundary. There are several sources of relevant additional information for this figure. The Appendix contains Figures \ref{fig:arch_diagrams} and \ref{fig:data_diagrams}, which are useful to quickly visually appreciate the architectures, ArchI,II,III, and training data, DataI,II,III. It also contains Figure \ref{fig:neuron_states}, which elaborate on these NSB diagrams in number, kind, and size.
  \ref{fig:arch_diagrams} and \ref{fig:data_diagrams}
    }
  \label{fig:unused_capacity}
\end{figure*}

\subsection{A Deep Think on Simple Observations}
\begin{displayquote}
\centering
We find in our experiments that DNNs of any architecture trained on linearly classifiable data are almost always linear classifiers (Fig \ref{fig:unused_capacity}).
\end{displayquote}

Is this interesting? Let us consider: though our network has enormous capacity, in this fixed setting of linearly separable data, the deep network behaves as though it has no more capacity than a linear model. When we discuss capacity of a class a functions, we ordinarily consider a hypothesis class consisting of networks indexed over all possible values of weights (or perhaps in a unit ball), since no such restrictions are explicitly built into in the learning algorithm. For a large architecture, such as {\Archiii}, this hypothesis class consists of a tremendous diversity of decision boundaries that fit the data. However, here we observe only a subset of learners: Not every configuration of weights nor every hypothesis is reachable by training with gradient descent on linearly classifiable data. Consider a learning the DNN weights corresponding to the $9$ layer network, {\Archiii}. The VCDim of such hypotheses indexed by every possible weight assignment is $1e6$, which is unhelpfully large. But, have we measured the capacity of the correct class? If we instead use the class reachable by gradient descent, then data assumptions, which are in some form necessary, by constraining the inputs to our learning algorithm in turn restrict our hypothesis class. Linear separability is a particularly strong data assumption which reduces our the VC dimension of our hypothesis class from $1e6$ to $3$. We conclude:

To ensure generalization of unregularized DNN learners, not only are data assumptions \textit{necessary}, but also strong enough assumptions on the training data are themselves \textit{sufficient} for generalization.

In Figures \ref{fig:1-b},\ref{fig:1-c}, 
we see that a DNN with more parameters learns a more complicated \textit{function} but not a more complicated \textit{classifier}. For example, the number of linear regions does seem to scale with depth for fixed dataset. However, instead of intersecting the decision boundary or one another, these additional NSBs form redundant onion-like structures parallel to the decision boundary.

Since we have argued that learning guarantees in this setting are essentially equivalent to training data guarantees, capacity measures on the learned network $\net$ that imply generalization must somehow reflect the regularity of the data that was originally trained on. Conversely, the factors not determined by the training data structure should \textit{not} factor into the capacity measure. For example, we desire bounds which do not grow with depth.

A capacity measure on $\net(x)$ that is determined entirely by restricting $\net$ to a neighborhood of its decision boundary accomplishes both such goals. The effect of the data is captured because the geometry of this boundary closely mirrors the that of the training data in arrangement and complexity. Consider also that behavior of $\net$ at the decision boundary is still is sufficient to determine each input classification and therefore the generalization analysis is unchanged. We claim  restricting $\net$ to near the decision boundary destroys the architecture information used to parameterize $\net$. More specifically, only the existence of neurons whose NSB intersects the decision boundary can be inferred from observation of inputs $x$ and outputs $\net(x)$ near the boundary. For example, this restriction is the same both for a linear classifier and for a deep network that learns a linear classifier with $50$ linear regions (as in Fig. \ref{fig:1-b}).




\section{Opening the Black Box through Deep Logical Circuits}
\label{sec:interpret}

\begin{figure*}[htb]
\def\fbfwidth{0.09\linewidth}
    \captionsetup[subfigure]{labelformat=empty}
   \centering
   \[\left[
    \begin{minipage}[h]{0.10\linewidth}
	\vspace{0pt}
	\includegraphics[width=\linewidth]{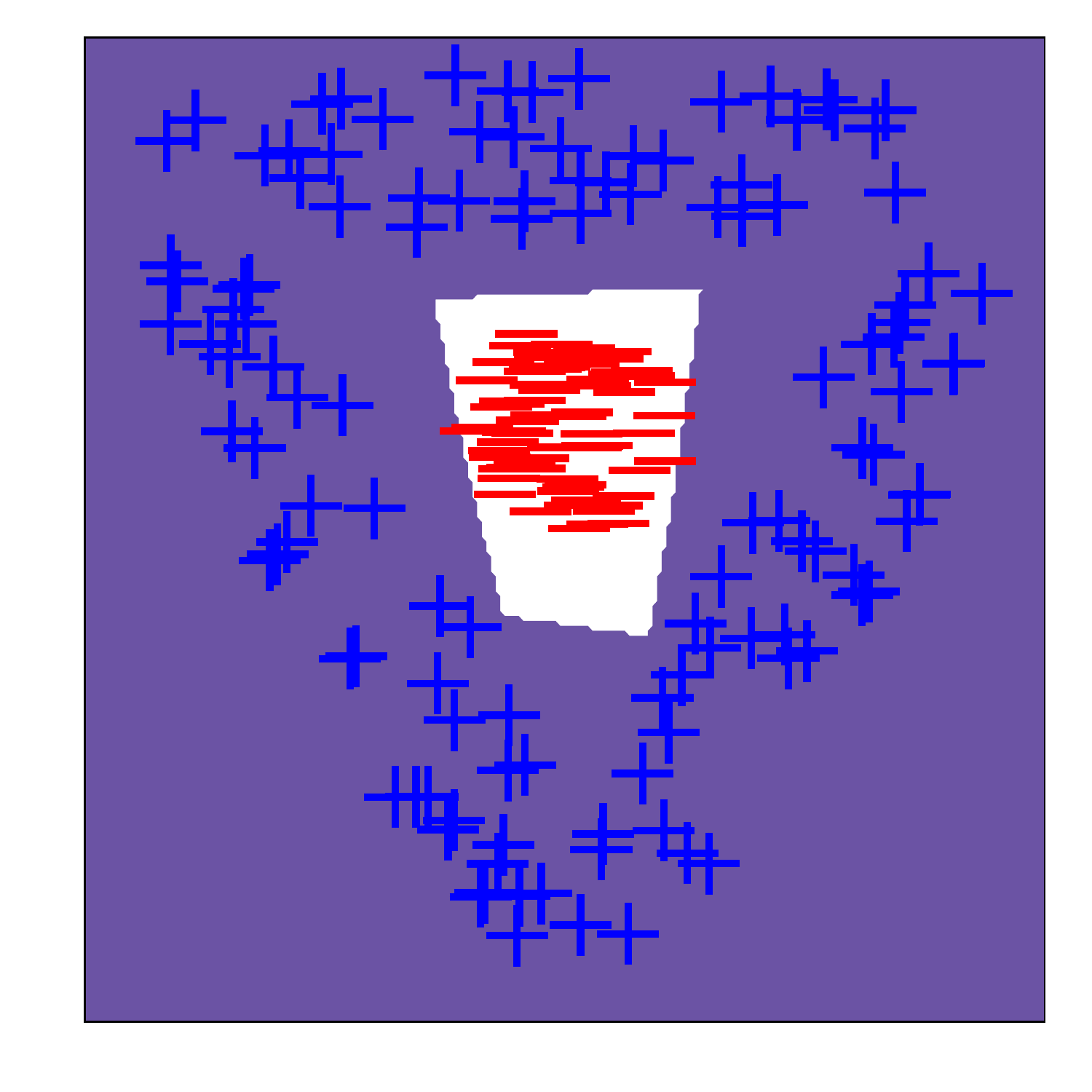}
	\subcaption{$\net(x)\geq 0$}
    \end{minipage}
   \right] 
   \Leftrightarrow
   \left[
   \begin{minipage}[h]{\fbfwidth}
	\vspace{0pt}
	\includegraphics[width=\linewidth]{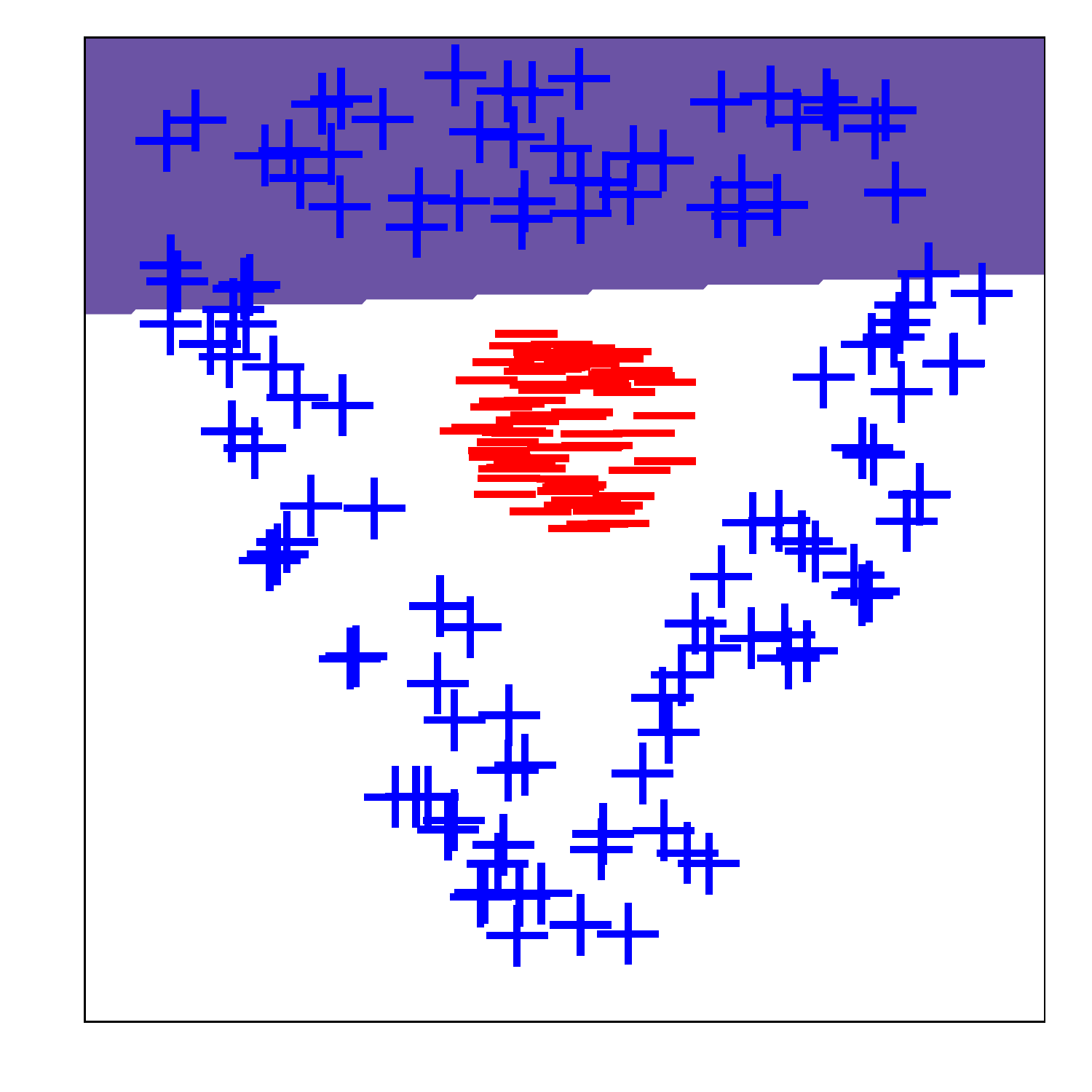}
    \end{minipage}
    \bigvee 
   \begin{minipage}[h]{\fbfwidth}
	\vspace{0pt}
	\includegraphics[width=\linewidth]{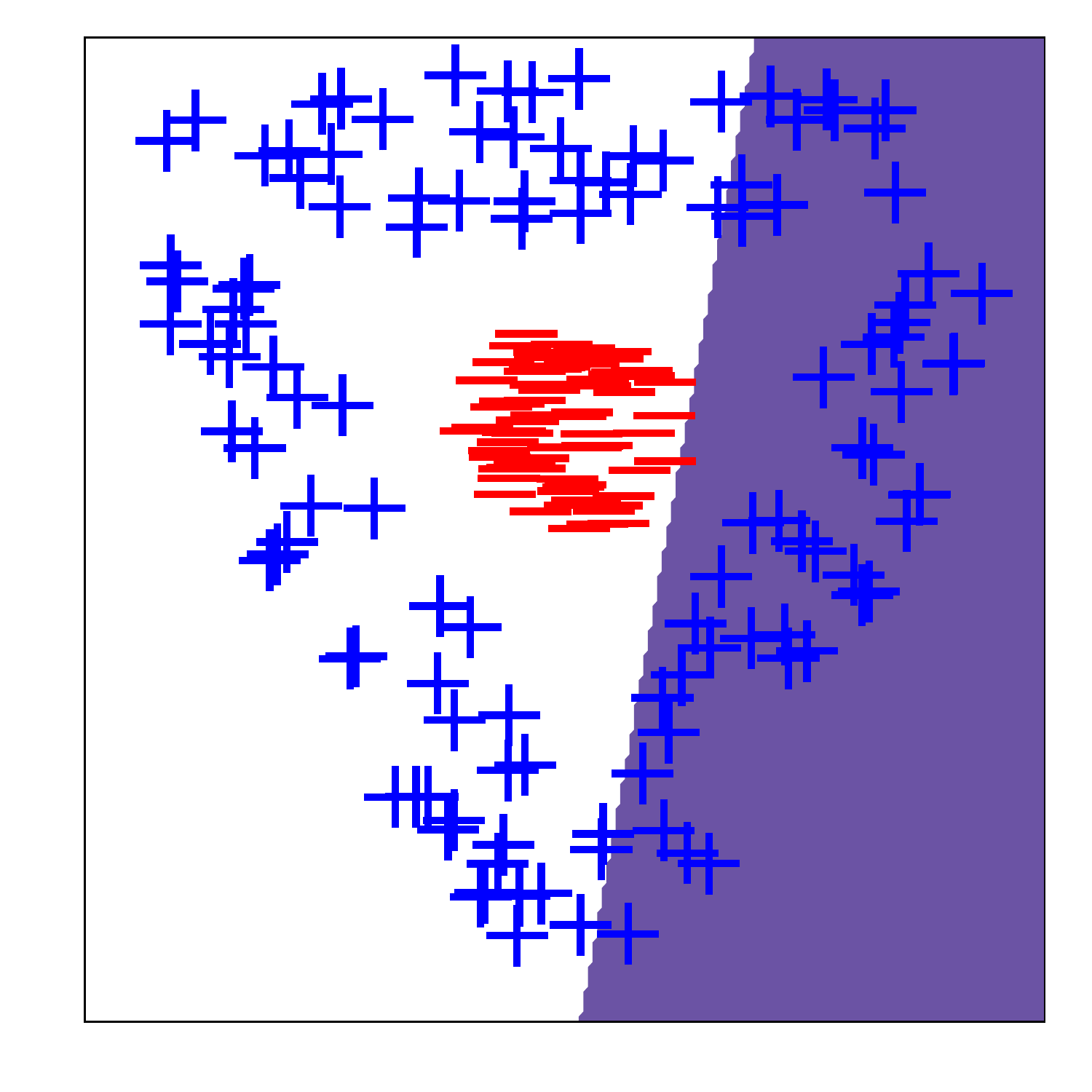}
    \end{minipage}
    \bigvee 
   \begin{minipage}[h]{\fbfwidth{}}
	\vspace{0pt}
	\includegraphics[width=\linewidth]{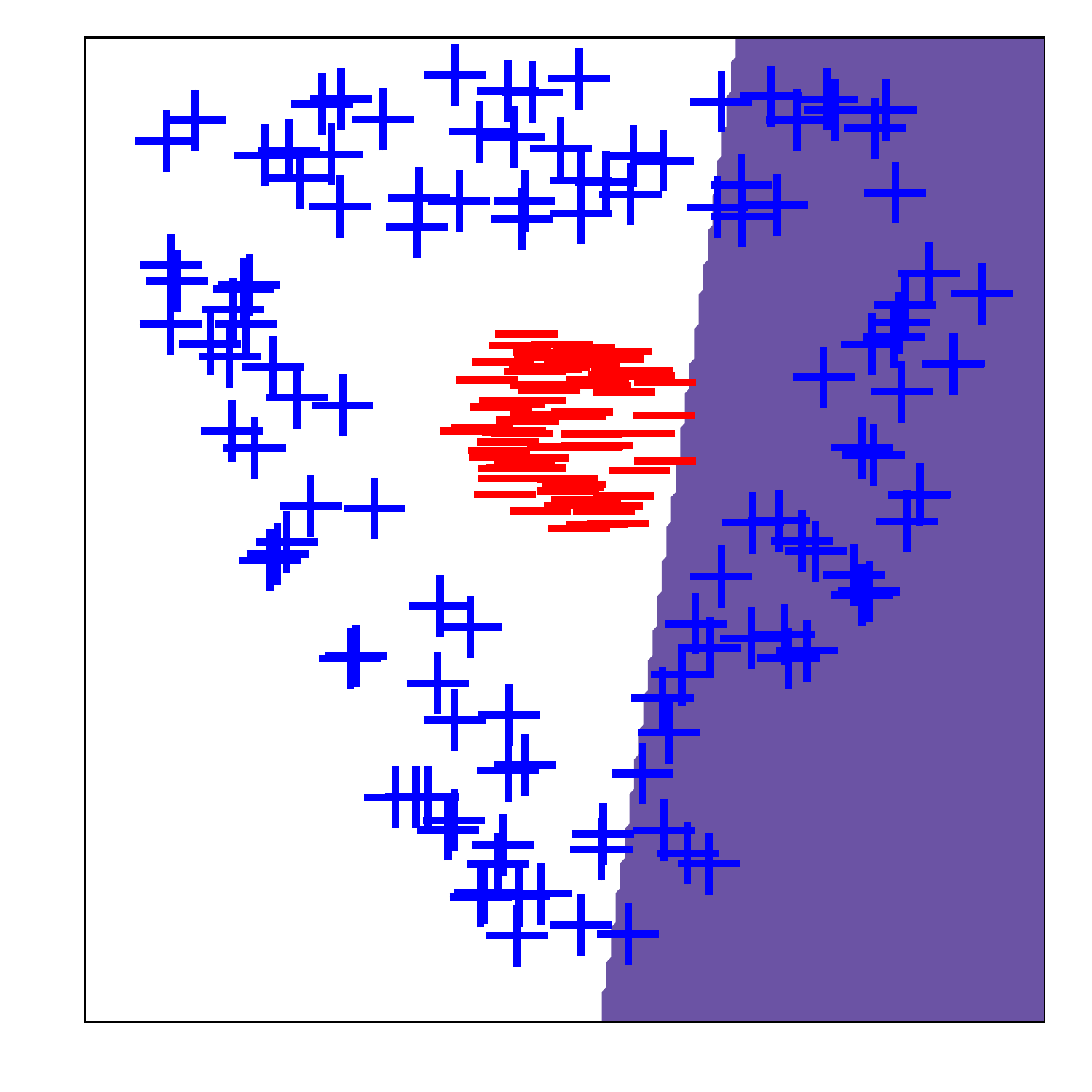}
    \end{minipage}
    \bigvee 
   \begin{minipage}[h]{\fbfwidth{}}
	\vspace{0pt}
	\includegraphics[width=\linewidth]{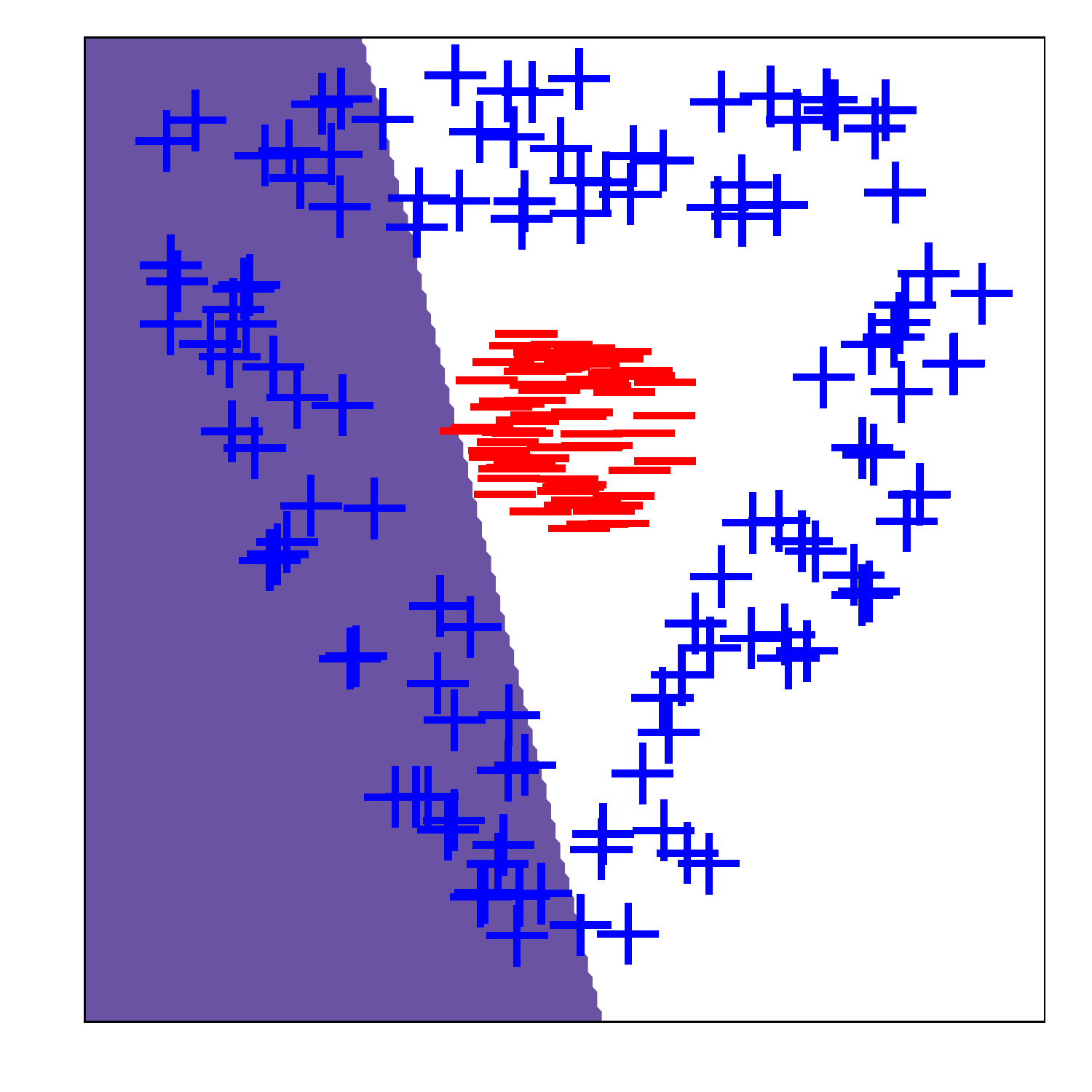}
    \end{minipage}
    \bigvee 
   \begin{minipage}[h]{\fbfwidth{}}
	\vspace{0pt}
	\includegraphics[width=\linewidth]{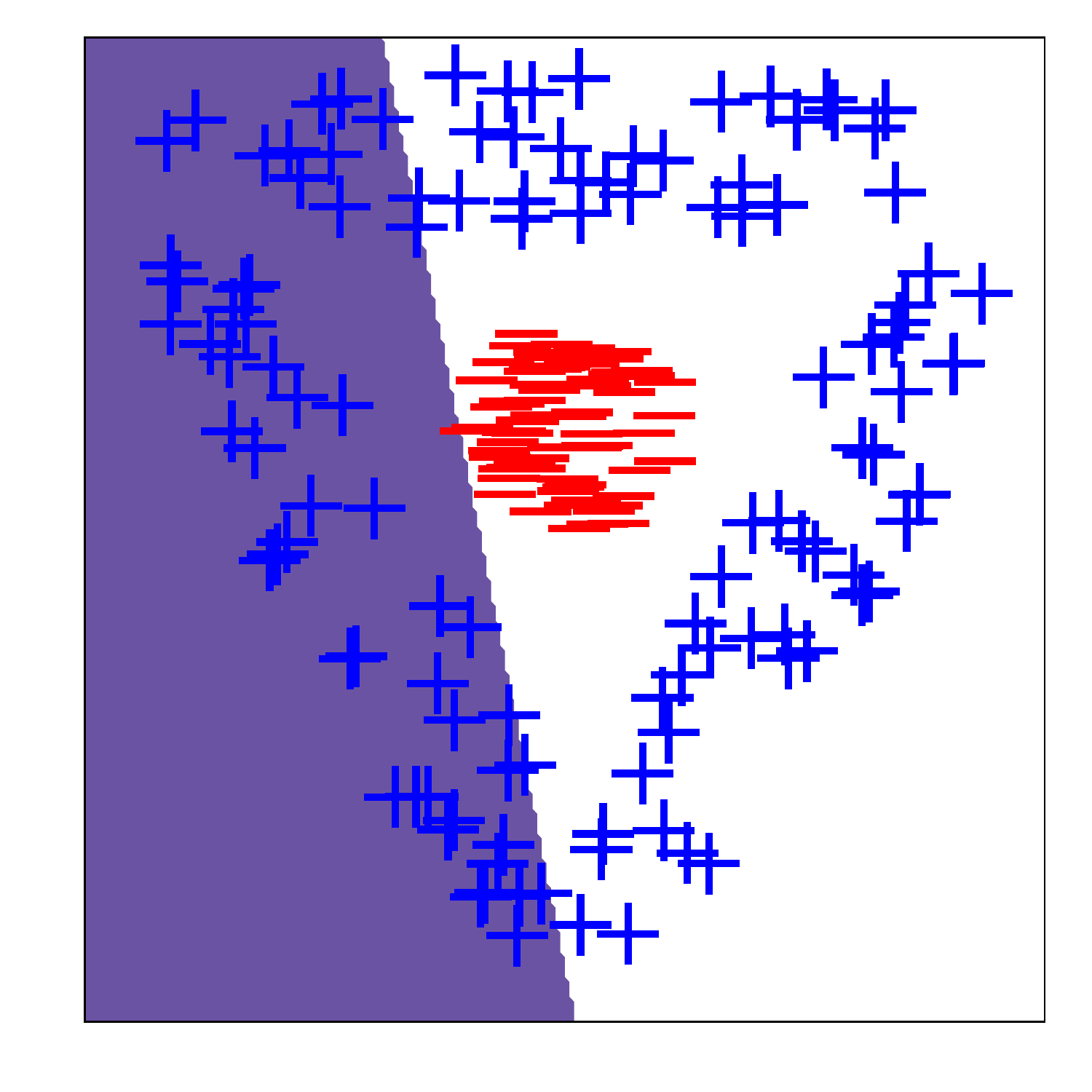}
    \end{minipage}
    \bigvee 
   \begin{minipage}[h]{\fbfwidth{}}
	\vspace{0pt}
	\includegraphics[width=\linewidth]{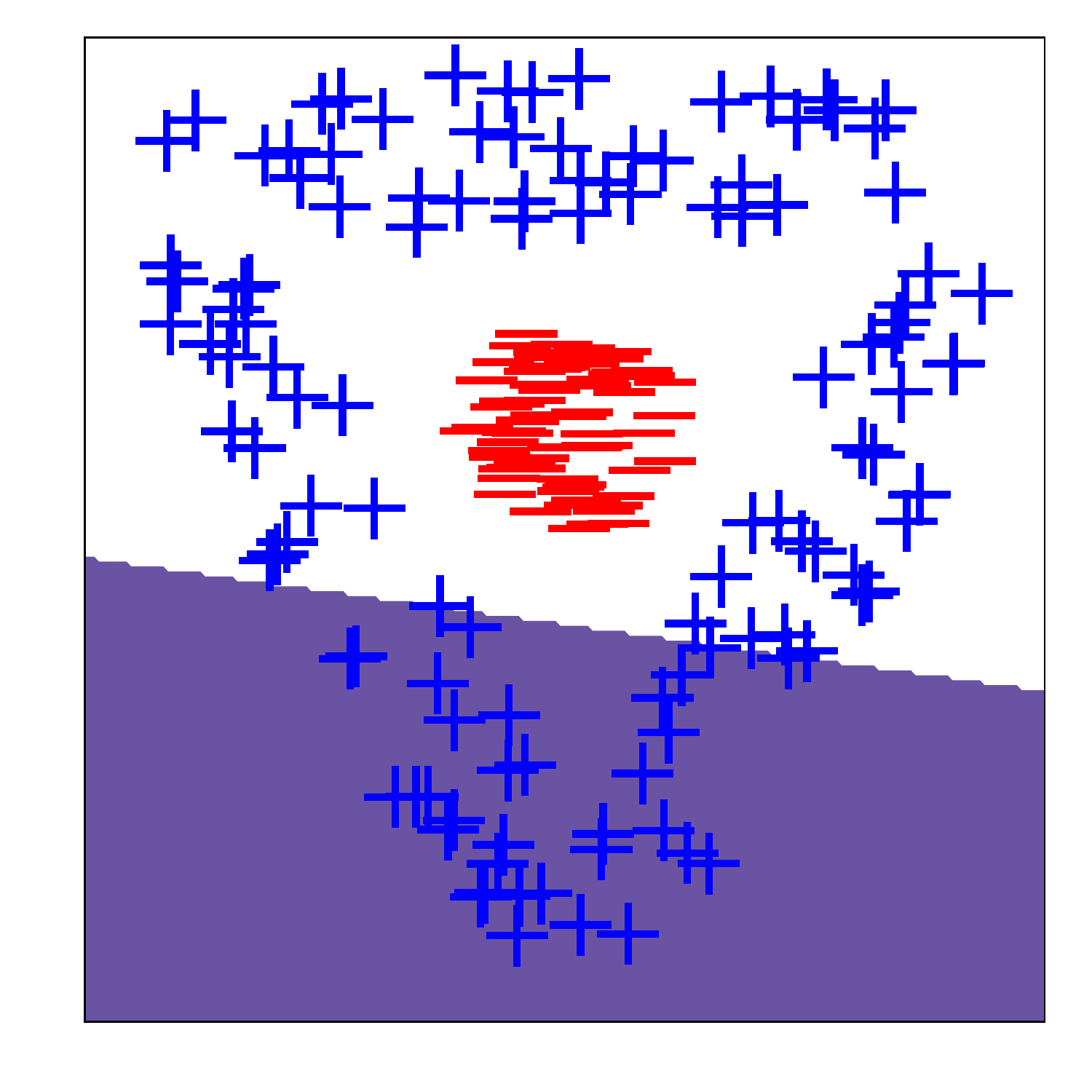}
    \end{minipage}
    \right]
   \]
   \caption{
   The logical representation of the classifier learned by the {\Archiii} network to classify the {\Dataii} data (shaded region classified True). Our algorithm outputs the RHS: the rules the DNN uses to assign a positive label. For at least one of the images on the right, the input must lie in the blue region. These rules are not apparent by inspection of the $\sim7e5$ network parameters.}
   \label{fig:bool_equiv_D2A3}
\end{figure*}

The key idea is to characterize the decision boundary of the DNN by writing the discrete valued classifier $x\mapsto\net(\x)\geq 0$ as a \textit{logical combination} of a hierarchy of intermediate classifiers. These intermediate classifiers identify higher order features useful for the learned task. The final result will be a logical circuit which produces the same binary label as the DNN classifier on all inputs. Finding a circuit that is "simple" is our key to both interpretability and generalization bounds. 

One such example is shown in Figure \ref{fig:bool_equiv_D2A3}. We show that our method translates a $1e^{6}$ DNN classification map into an OR combination of just $6$ linear classifiers. To functionally emphasize, our representation \textit{is} the DNN. It applies to all inputs: training, test, and adversarial alike.

\begin{figure*}
  \centering
         \begin{subfigure}[b]{\linewidth}
                \centering
                \includegraphics[width=\textwidth]{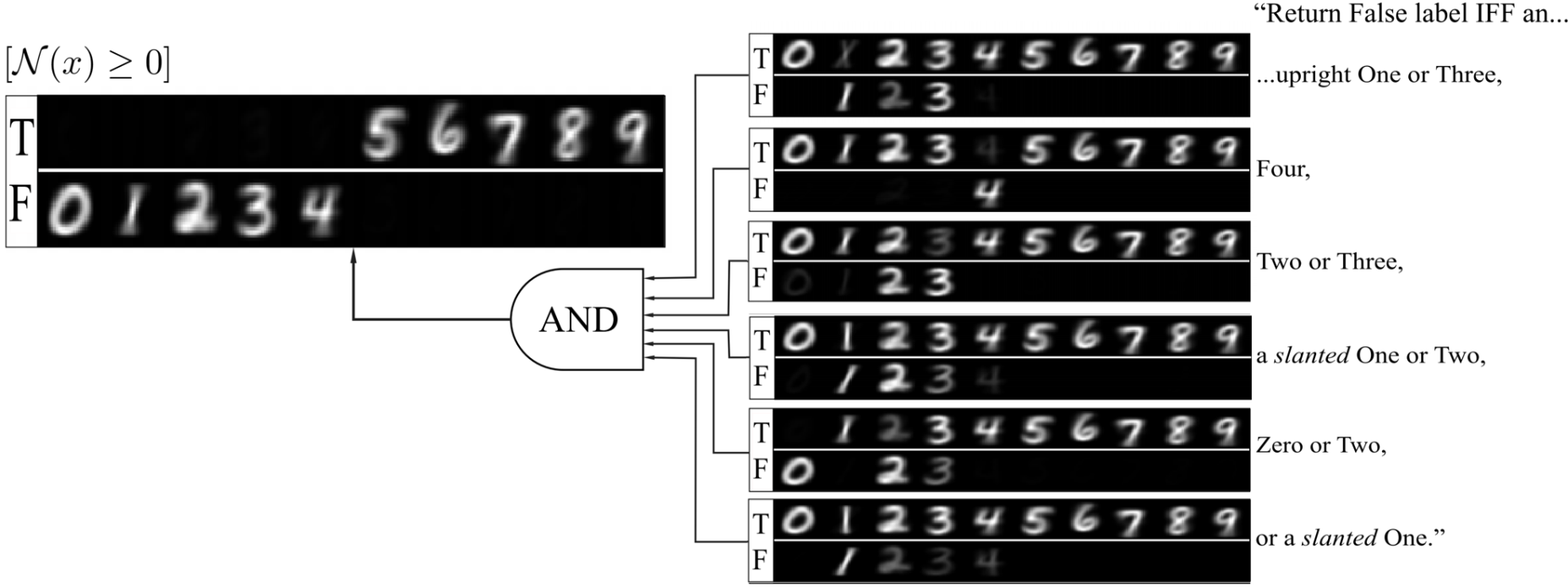}
                \caption{Trained DNN with a Concise "English" Description}
                \label{fig:mnist-a}
         \end{subfigure}
        \vspace{.18cm}

          \begin{subfigure}[b]{\linewidth}
                \centering
                \includegraphics[width=\textwidth]{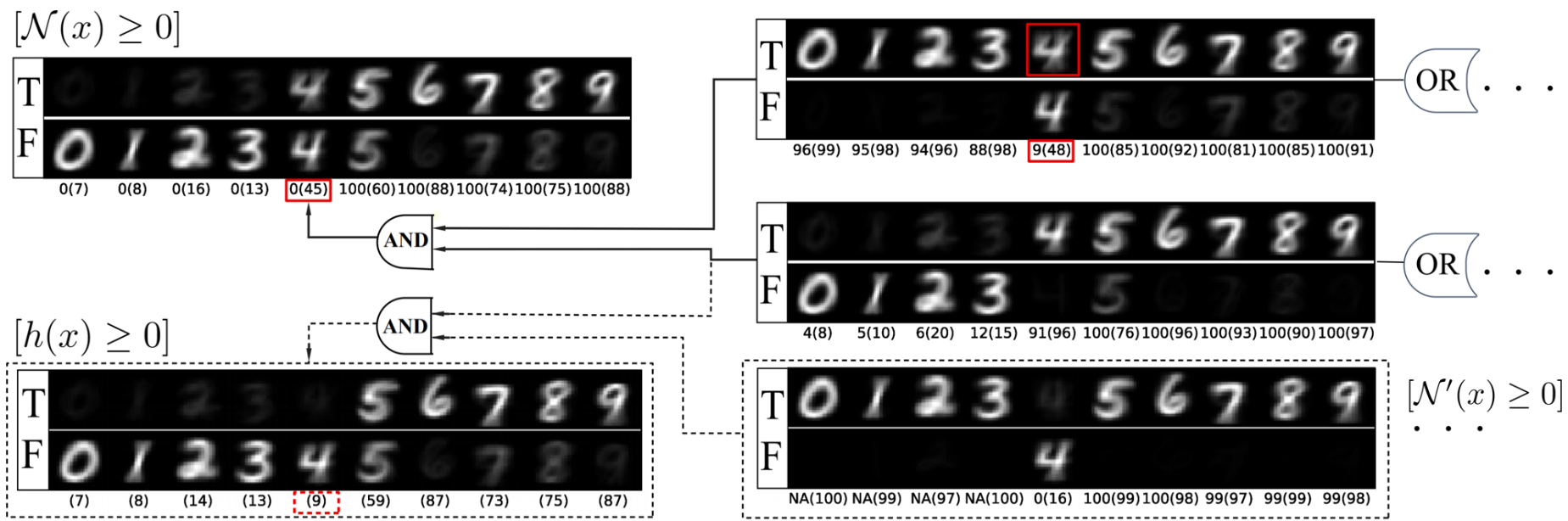}
                \caption{A circuit with localized memorization}
                \label{fig:mnist-c}
     \end{subfigure}
         
  \caption{Selected subsets of the logical circuits corresponding to binary classification DNNs trained on the MNIST dataset. Each $2\times10$ array represents a different binary classifier within the network circuit, which assigns True or False to every input image. In Fig. \ref{fig:mnist-a}[Fig. \ref{fig:mnist-c}] training[test] images of number $l$ contribute in the $l^{th}$ column to the brightness of either top or the bottom row of every array. The choice of row corresponds to whether corresponding classifier outputs True or False. The diagram reads right to left along solid arrows terminating in the leftmost array corresponding to the DNN binary output $[\net(\x)\geq 0]$. 
  The training objective only explicitly distinguishes $0-4$ from $5-9$. Yet, we see that the intermediate logical computations the network learns delineate semantically meaningful subcategories. In Fig. \ref{fig:mnist-a}, the DNN internal logic even admits a description \textit{in plain English}. 
  We show in Figure \ref{fig:mnist-c} how the internal logical circuitry \textit{within the DNN} can be tweaked to improve generalization. Connected with solid lines, we see a network that has overfit badly ($1.0,0.78$ train and test accuracy). The percentage [in parenthesis] under each column indicates how that training[test] digit is assigned True. We see that the intermediate classifier (middle right) struggles to separate Four from the positive labels. A second classifier (top right) is dedicated to learning to identify Fours as False, allowing the network to fit the training data. However, by comparing training and test performance, we can see that these Fours are not learned but memorized: As shown in solid rectangles, the intermediate and final classifier, respectively, assign True to $9\%$[$48$\%] and $0\%$[$45\%$] of the Fours in the training[test] set, accounting for the bulk of the generalization error! In practice, this network would be discarded and retrained from scratch. Since we now have access to the internal logic of the network, we are instead able to surgically replace the memorizing component. The first step we have done implicitly: we use \textit{domain knowledge} to interpret the component function as "excluding Fours". We then train a second network, we call a "prosthetic", learning $[\net^{\prime}(x)\geq 0]$ (with the same settings and data), to label $4$ as False and $5-9$ as True. We can then excise the memorizing component, replacing its role in the logical circuit with the prosthetic (bottom right) to obtain a new classifier consisting of the three classifiers connected by dotted lines. The classifier we engineer ($[h(x)\geq 0]$ bottom left) does better on Fours, 
  $45\%\rightarrow9\%$ classified True (dotted rectangle) and has higher test accuracy overall ($.78\rightarrow.83$).
  }
  \label{fig:mnist}
\end{figure*}

When we train networks on the MNIST dataset, the learned circuit is more complicated, but we can still understand "role" of the intermediate classifiers within the circuit. By probing the internal circuitry with training and validation inputs, we can interpret the role of the components by  cross-referencing with semantic categories (perhaps provided by a domain expert). \textit{A priori}, there is no reason why this should be possible: The high level features a DNN learns as useful for this task are not obliged to be those that humans identify. However we see experimentally extremely encouraging evidence for this. When we group digits $0-4$ and $5-9$ into binary targets for classification, the DNN virtually always learns individual digits as intermediate steps within the logical circuit (Figure \ref{fig:mnist}). For space, only those circuit components closest to the output are shown. A more involved circuit study is available in Figure \ref{fig:mnist-b}.

The dichotomy presented is that Fig \ref{fig:mnist-a} demonstrates the importance of our method to interpretability, while Fig \ref{fig:mnist-c}, demonstrates the importance toward improving generalization. Although interpretability and generalization are usually studied separately, understanding "what $\net$ has learned" is actually very closely related to understanding "what $\net$ has memorized". In fact, one of the takeaways from Figure \ref{fig:mnist} is that the \textit{mechanism of memorization} itself can have interpretation. In Figure \ref{fig:mnist-c} we exploit such an interpretation to improve generalization error by "repairing" the defect.



\def\TheorySection{A Theory of DNNs as Logical Hierarchies}
\section{\TheorySection{}}
In this section, starting with any fixed DNN classifier, we show how to construct, simplify, measure complexity of, and derive generalization bounds for an equivalent logical circuit. These bounds apply to the original DNN. We show they compare favorably with traditional norm based capacity measures.

\subsection{Boolean Conversion: Notation and Technique}
In our theory, we designate $\mu$ and $\tau$ as special characters with dual roles and identical conventions. We consider the symbols, $\mu,\tau$, to represent binary vectors that index by default over all binary vectors
and implicitly promote to diagonal binary matrices, $Diag(\mu), Diag(\tau)$, for purposes of matrix multiplication. For matrices, $M$, we define $(M_{\pm})_{i,j}=\max\{0,\pm M_{i,j}\}$.
To demonstrate, we have for $d=1$: 
$\net(\x)=\B{1}+\max_\mu \W{1}_+\mu(\B{0}+\W{0}\x) - \max_\tau \W{1}_-\tau(\B{0}+\W{0}\x)$. In fact, we may write this as a MinMax or a MaxMin formulation by substituting, $-\max_\tau=\min-\tau$, and factoring out the Max and Min in either order. Our primary tool to relate to Boolean formulations is the following.

\begin{restatable}{proposition}{PropOpEquiv} \label{prop:minmax2bool}
Let $f:\mathcal{A}\times\mathcal{B}\mapsto \bbR$. Then we have the following logical equivalence:
\begin{equation*}
\biggl[ \max_{\alpha\in\mathcal{A}} \min_{\beta\in\mathcal{B}} f(\alpha,\beta)\geq 0 \biggr] \Longleftrightarrow \bigvee_{\alpha\in\mathcal{A}} \bigwedge_{\beta\in\mathcal{B}} \biggl[ f(\alpha,\beta)\geq 0 \biggr]
\end{equation*}
\end{restatable}

We classify network states, $\Sigbar=\SigbarP\cup\SigbarM$, in terms of the output sign,
$\Sigbar_{\pm}=\{\sigbar(\x)|\pm\net(\x)\geq 0\}$. We use $\SigbarO=\SigbarP\cap\SigbarM$ for those states at the boundary. For $J\subset[\depth]$, we define $\Sigbarl{J}$[$\SigbarOl{J}$] to be the projection of $\Sigbar$[$\SigbarO$] onto the coordinates indexed by $J$. As a shorthand, we understand the symbols $\mubar^{[k]}=(\mul{1},\ldots,\mul{k})$ and $\mubar=\mubar^{[d]}=(\mul{1},\ldots,\mul{d})$ to be equivalent in any context they appear together.



Define for every $\tau,\mu$, a linear function of $\x$, $\Recurse{1}(\mu,\tau,\x)=\B{1}+\W{1}_{+}\mul{1}(\B{0}+\W{0}\x) - \W{1}_{-}\taul{1}(\B{0}+\W{0}\x)$, 
called the "Net Operand". We have 
$[\net(x)\geq 0]\Leftrightarrow \vee_\mu\wedge\_\tau [\Recurse{1}(\mu,\tau,\x)\geq 0]$. To generalize to more layers, we can recursively define:
\begin{align*}
    &\Recurse{l+1}(\mubar^{[l+1]},\taubar^{[l+1]}, \x)=
    \W{l+1}_{+}\mul{l+1}\Recurse{l}(\mubar^{[l]},\taubar^{[l]}, \x)\\
    &\quad -\W{l+1}_{-}\taul{l+1}\Recurse{l}(\taubar^{[l]},\mubar^{[l]}, \x)
    +\B{l+1}.\\
\end{align*}
One can derive by substitution that $\DotOp{\sigbar(\x),\sigbar(\x)}{\x}=\net(\x)$. This choice of $\mubar=\taubar=\sigbar(\x)$ will always be a saddle point solution to Eqn $1$ in the following theorem.
\begin{restatable}{theorem}{CompositionalTheorem}\label{thm:hierarchical}
Let $\Op$ be the net operand for any fully-connected ReLU network, $\net$. 
Then,
\begin{align}
    \net(\x)&=\max_{\mul{d}}\min_{\taul{d}}\cdots
    \max_{\mul{1}}\min_{\taul{1}} 
    \DotOp{\mubar,\taubar}{\x}                  \label{eqn:net_minmax}
    \\
    \biggl[ \net(x)\geq 0 \biggr]&\Leftrightarrow \bigvee_{\mul{d}}\bigwedge_{\taul{d}} 
    \cdots \bigvee_{\mul{1}}\bigwedge_{\taul{1}} \biggl[ 
    \DotOp{\mubar,\taubar}{\x}\geq 0
    \biggr]  \label{eqn:net_orand}
\end{align}
\end{restatable}
Notice that we can derive the second line (2) from the first (1) by recursive application of Proposition \ref{prop:minmax2bool}. 
Since we index over all binary states, the number of terms in our decomposition (Eqn 2) is extremely large. 
Though (it turns out) we may simply matters considerably by indexing instead over network states, $\Sigbar$. The next Theorem says that when the right hand side(RHS) of Eqn. $1$ is indexed by only those states realized at the decision boundary, $\SigbarO$, the RHS still agrees with $\net(x)$ in \textit{sign}, but necessarily numerical value. Thus they are equivalent classifiers. 

\begin{restatable}{theorem}{BoundaryTheorem}
\label{thm:sig_bdry}
Let $\net$ be a fully-connected ReLU network with net operand, $\Op$, and boundary states, $\SigbarO$. Then,
\begin{align}
&[\net(x)\geq 0]\Leftrightarrow
    \bigvee_{\mul{d}\in\SigbarOl{d}}
    \bigwedge_{\taul{d}\in\SigbarOl{d}}
    \bigvee_{\{\mul{d-1}|(\mul{d-1},\mul{d})\in\SigbarOl{d-1,d}\}}
    \cdots\nonumber\\
&\quad
    \bigvee_{\{\mul{1}|\mubar\in\SigbarO\}}
    \bigwedge_{\{\taul{1}|\taubar\in\SigbarO\}}
    \biggl[
    \DotOp{\mubar,\taubar}{\x}\geq 0
    \biggr]
\end{align}\label{eqn:EquivSigO}
\end{restatable}

The proofs for both Theorems \ref{thm:hierarchical} and \ref{thm:sig_bdry} can be found in the Appendix \ref{sec:proofs}. 
We also include explicit pseudocode, "Network Tree Algorithm" \ref{alg:op_tree} (in Appendix \ref{app:algorithms}) for constructing our Logical Circuit from $\SigbarO$. Somehow, we find the actual python implementation more readable, which we have included in the supplemental named "network\_tree\_decomposition.py". We use this file to generate the readout in Figure \ref{fig:es-d} (Appendix \ref{sec:exp-support}) to provide tangible,  experimental support for the validity of our conversion algorithm.

\def\GeneralizationSection{Formalizing Capacity for Logical Circuits}
\subsection{\GeneralizationSection{}}
\label{sec:gen_bound}

We repurpose the following theorem used by \citep{Bartlett2017b} for ReLU networks data-independent VC dimension bounds. 

\begin{restatable}{theorem}{Jerrum} (Theorem 17 in \citep{Goldberg1995a}):
Let $k$,$n$ be positive integers and $f:\bbR^n\times\bbR^k\mapsto\{0,1\}$ be a function that can be expressed as a Boolean formula containing $s$ distinct atomic predicates where each atomic predicate is a polynomial inequality or equality in $k+n$ variables of degree at most $d$. Let $\mathcal{F}=\{f(\cdot,w):w\in\bbR^k\}$. Then $\VCvap(\mathcal{F})\leq 2k\log_2(8eds)$.
\label{thm:Jerrum}
\end{restatable}

\def\ksdBool{(k,s,d)-Formula}
As a short hand, we refer to any Boolean formula satisfying the premises in the above theorem as class $(k,s,d)$. If we consider (for fixed $\SigbarO$) the complexity of learning the weights defining the linear maps in Eqn. \ref{eqn:EquivSigO}, Jerrum's Theorem tells us that we are primarily concerned with the number of parameters being learned. Fortunately, we only pay a learning penalty for those weights distinguishable by neuron activations in $\SigbarO$. For example, within the same layer, a single neuron is sufficient model any collection of neurons which are always "on" or "off" simultaneously at the decision boundary. In general, we can restrict to a subset of $r_l=rank(\SigbarOl{l})$ representative neurons without sacrificing expressivity at the boundary. 
We can additionally delete entire layers when $r_l=1$. 
We use $\bar{r}\triangleq\defrankbar$ to group the dimensions of the reduced architecture into a single vector. Note that when $\net$ is a linear classifier, then $\bar{r}$ is a vector of all $1$s. $r_l=1$ at every layer. 

Finally, we define
$\Phi(\net):\bbR^k\times\bbR^{\width{0}}$ to be the Boolean function in Eqn \ref{eqn:EquivSigO} corresponding to the \textit{reduced} network, whose depth we also overload as $d$, and take $k$ to be the number of parameters on which the formula depends. The formula has $s=|\SigbarO|^2$ inequalities. The explicit calculations for determining $k,s,d,\bar{r}$ are described Function $MinimalDescrip$ in Algorithm \ref{alg:gen_bound} in Appendix \ref{app:algorithms}. The following is automatic given the discussion so far.

\def\HypBool{\mathcal{H}_{\Phi(\net)}}
\def\defHypBool{\{ x\mapsto\Phi(\net)(w,x) | w\in\bbR^k\}}
\begin{restatable}{theorem}{VCBoolTheorem}
\label{thm:vcbool}
Let $\net:\bbR^{\width{0}}\mapsto\bbR$ be a fully-connected ReLU network. Suppose the Boolean formula, $\Phi(\net)$, is of class $(k,s,d)$. Define the hypothesis class $\HypBool\triangleq\defHypBool{}$. Then 
\begin{enumerate}
    \item $x\mapsto[\net(x)\geq 0]\in\HypBool{}$
    \item $\VCvap(\HypBool{})\leq 2k\log_2(8esd)$
\end{enumerate}
\end{restatable}

Of course, this bound only applies to the learned DNN if the hypothesis class $\HypBool{}$ is implied in advance. To address (informally) the capacity for a single classifier, $\net$, we define $\VCbool(\net)\triangleq \defVCBool{}$ to be the complexity of  learning the parameters of the $(k,s,d)-$Boolean formula representing $\net$. This is an upper bound for the smallest complexity over formulae $\Phi$ and classes $\HypBool{}$ containing $\net\geq 0$ as a member. In Figure \ref{fig:gen_bound}, we train $\net$ to classify samples in \Dataiii{} and compare qualitatively 
our capacity measure, $\VCbool(\net)$, with those of other well-known approaches as we vary the network size and training duration and depth. We compare with methods which appear at first glance to make use of additional information---that of scale, norm, and margin---which should in principle produce tighter bounds. 

And yet, we enjoy a comfortable edge over other comparable methods. Under all conditions, our bound seems to be orders of magnitude smaller than these other (well-respected) bounds. So, what is going on? In fact, it is \textit{our} bound that is advantaged by using more (between-layer) information!. 

We revisit the observation that a very deep DNN trained on linearly separable data is a linear classifier. We think that this simple characterization should somehow be accessible to our capacity measure through the weights. Linearly separable data represents, to us, the simplest, plausible, real-world proving ground for models of DNN generalization error. The methods with which we compare bound the distortion applied by each layer in terms of a corresponding weight matrix norm and accumulate the result. We should like our method to "realize" that the DNN classifier is linear, but this can not be discovered by scoring each layer. In fact, having an efficient Boolean representation is a \textit{global} property that is sensitive to the relative configuration of weights across all layers. It is not information that is contained in the weight norms used by other methods, which destroy weight-sign information, among other properties, on which linearity of the classifier depends. We would even suggest that our notion of regularity is "more nuanced" in the sense that whether a layer is well-behaved only makes sense to talk about within the context of the overall network.

Returning to Figure \ref{fig-b}, we observe that we our bound is relatively stable with respect to increasing architecture size and depth. This behavior is instructive in its distinction from that of uniform (data-independent) VC dimension bounds, $\VCind$, which depend on the architecture alone. That these bounds produce unreasonably large, vacuous bounds for over-parameterized models is widely known and often recited. 
Perhaps this notoriety 
has dissuaded combinatorial analyses of DNN complexity altogether. 
However, our results demonstrate 
that the vast majority of the bloat in these $\VCind$ bounds can be attributed to a lack of strong data assumptions and not to its combinatorial nature. When we compare against our own (also combinatorial) measure, $\VCbool$, in Table \ref{tab:comb_stat} we observe  that $\VCbool$ produces bounds that are orders of magnitude smaller. We account for this discrepancy as follow: While $\VCind$ yields weak bounds on generalization that always apply, $\VCbool$ instead produces strong bounds that apply only when the data is nice. These bounds are smaller because the set of DNNs achievable by gradient descent on nice data is much more regular, and of smaller VC dimension. We explore this comparison in more depth in Appendix \ref{sec:bartlett_compare}.

Lastly, we offer some perspectives connecting our generalization studies to building better models in the future. 
There are many descriptions of complexity for DNNs. What makes ours a "good" one? All are equally valid in the sense that \textit{every} one of them can prescribe some sufficiently strong regularity condition that will provably close the gap between training and test error. But, perhaps we should be more ambitious. We actually want to decrease model capacity \textit{while also} retaining the ability to fit those patterns "typical" of real world data. 
While this second property is critical, it is also completely unclear how to guarantee, even analyze, or even define unambiguously. 
We surmise that since our capacity measure $\VCbool{}$ seems empirically to be \textit{already} minimized when the data is sufficiently structured, we can hope (and plausibly hypothesize) that those patterns that can be learned efficiently by a function class where $\VCbool{}$ is controlled \textit{explicitly} will not differ from those suited to unregularized DNNs, where we expect the structured nature of real world data to \textit{implicitly} regulate $\VCbool{}$ already.



\def\fght{5cm}
\begin{figure*}
\begin{subfigure}[t]{0.625\textwidth}
\hspace{1cm}
\includegraphics[height=\fght]{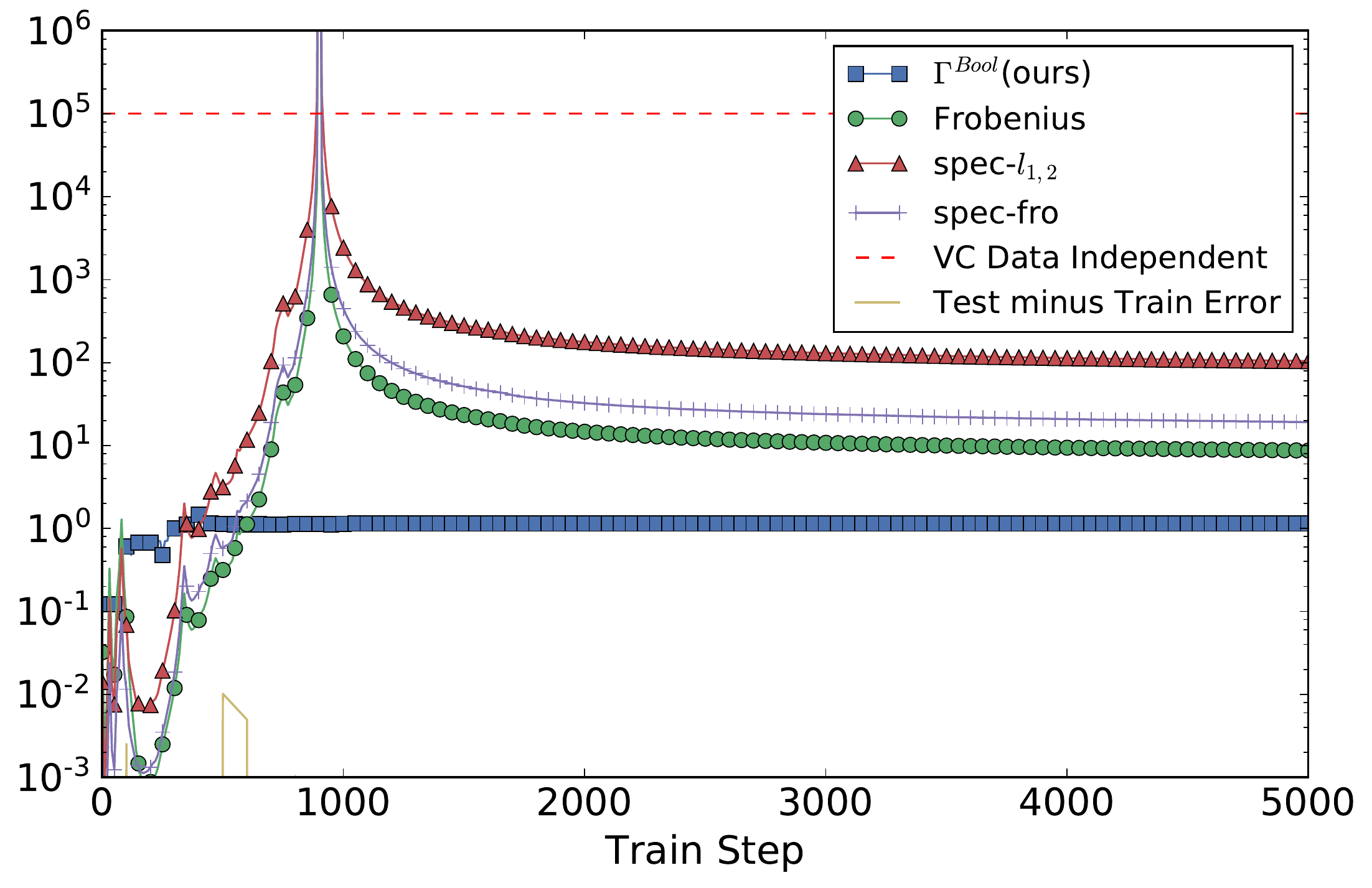}
\caption{Capacity vs Training Step}
\label{fig-a}
\end{subfigure}
\begin{subfigure}[t]{0.375\textwidth}
\hspace{0.4cm}
\includegraphics[height=\fght]{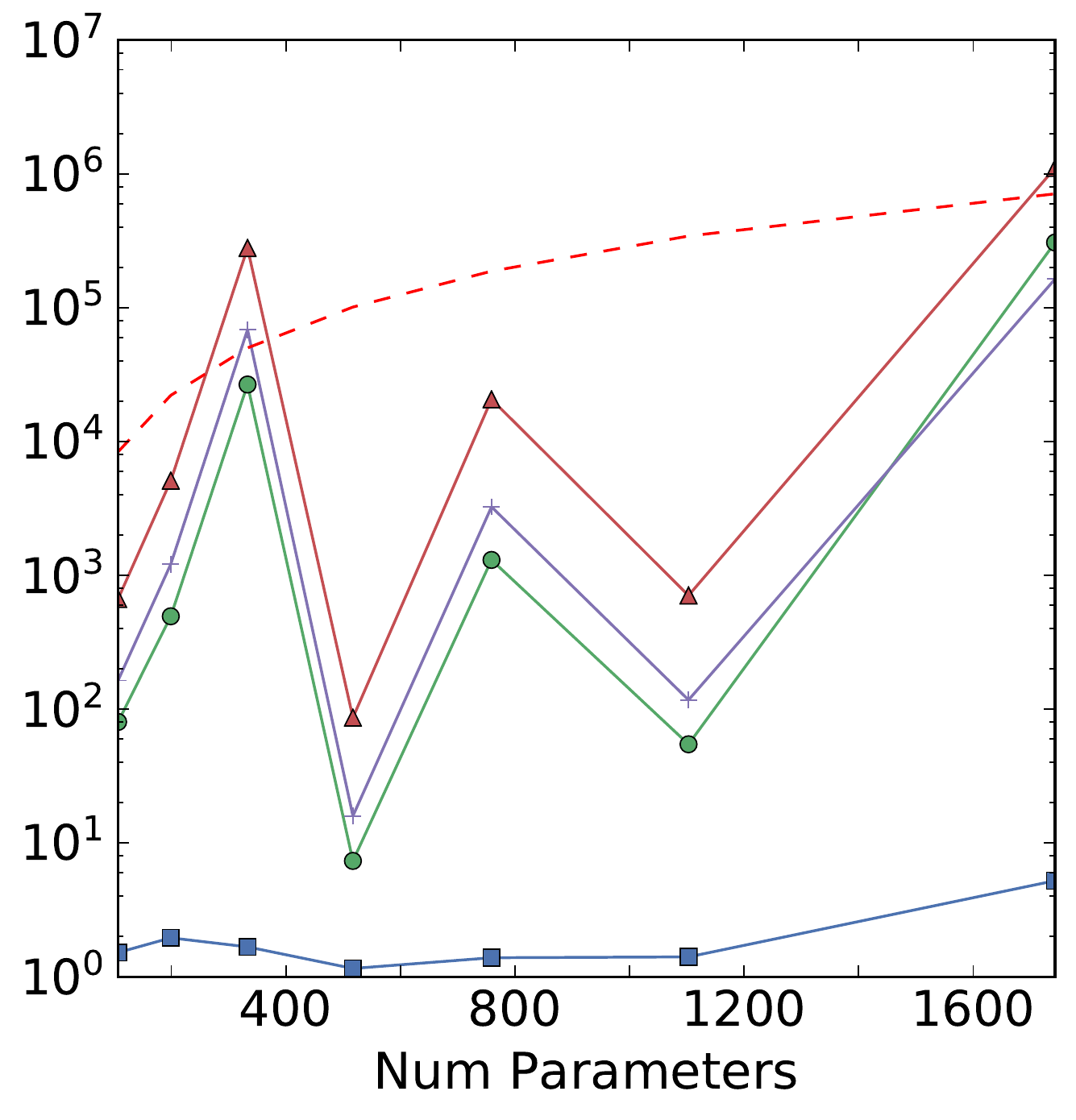}
\caption{Capacity vs Num Parameters}
\label{fig-b}
\end{subfigure}
\caption{
Qualitative comparisons of bounds on the generalization error for networks trained on {\Dataiii} during training (Fig. \ref{fig-a}) and as additional layers are added (Fig. \ref{fig-b}). Though our bound is in terms of VC dimension only, we compare favorably with other bounds that additionally use margin. 
Interestingly, the spike in capacity that occurs around $1000$ training steps is not reflected in our bound, but captured by others. Thus, our method may be blind to some interesting training dynamics, for example, a massive shift in the relative alignment of weight vectors that leaves the intersection system of neuron state boundaries unchanged. 
The empirical phenomenon of depth-invariant generalization error is consistent with the behavior of our bound (Fig. \ref{fig-b}). These trends are representative of all $9$ experiments (Figure \ref{fig:all_bounds}). 
\\
\\
\\
Frobenius: $\frac{1}{m}\frac{1}{\gamma^2}\prod_{l=1}^d\norm{\W{l}}_{F}^2$\citep{Neyshabur2015}\\
spec-$l_{1,2}$:$\frac{1}{m}\frac{1}{\gamma^2}\prod_{l=1}^d\norm{\W{l}}_{2}^2\sum_{l=1}^d \frac{\norm{\W{l}}_{1,2}^2}{\norm{\W{l}}_2^2}$\citep{Bartlett2017}\\
spec-fro: $\frac{1}{m}\frac{1}{\gamma^2}\prod_{l=1}^d\norm{\W{l}}_{2}^2\sum_{l=1}^d h_l \frac{\norm{\W{l}}_{F}^2}{\norm{\W{l}}_2^2}$\citep{Neyshaburb}\\ 
$\Gamma^{Bool}$(ours):$\min_{k,s,d} \sqrt{\frac{VC^{Bool}_{k,s,d}(\net)}{m}}$
  .}
  \label{fig:gen_bound}
\end{figure*}



\section{Related Work}
Our discussion of the role the data plays in generalization is perhaps most similar to 
\citet{Arpit2017}. 
Many authors have studied the number of linear regions of a DNN before, usually focusing on a $2$D slice or path through a the data
\citep{Serra2018,
Raghu,
Aroraa}
, optionally including study of how these regions change with training
\citep{Hanin,
Novak2018a} 
or an informal proxy for network "complexity"
\citep{Zhang2018,
Novak2018a}


Formal approaches to explain generalization of DNNs fall into either "direct" or "indirect" categories.
By direct, we mean that the bounds apply exactly to the trained learner, not to an approximation or stochastic counterpart. Ours falls under this category, so these are the bounds we compare to, including \citep{Neyshabur2015,Bartlett2017,Neyshaburb}, which we compare to in Fig. \ref{fig:gen_bound}.
While our approach relies on bounding possible training labelings (VCdim), these works all rely on having small enough weight norm compared to output margin.


Indirect approaches analyze either a compressed or stochastic version of the DNN function. For example, PAC-Bayes analysis \citep{McAllester1999} 
of neural networks 
\citep{Langford2002,Dziugaite2017}
produces uniform generalization bounds over distributions of classifiers which scale with the divergence from some prior over classifiers. 
Recently, 
\citet{Valle-Perez2018} 
produced promising such PAC-Bayes bounds, but they rely on an assumption that training samples the zero-error region uniformly, as well as some approximations of the prior marginal likelihood. Interestingly, they also touch on descriptional complexity in the appendix, which is thematically similar to our approach, but do not seem to have an algorithm to produce such a description. 
Another popular approach is to study a DNN through its compression 
\citep{Arora2018}
\citep{Zhou2018}
. Unlike our approach, which studies an equivalent classifier, these bounds apply only to the compressed version.

\section{Conclusions}
The motivation for our investigation was to describe regularity from the viewpoint of "monotonicity". Suppose that during training, the activations of a neuron in a lower layer separate the training data. 
While the specifics of gradient descent can be messy, there is no "reason" to learn anything other than a monotonic relationship (as we move in the input space) between the activations of that neuron, intermediate neurons in later layers, and the output. Two neurons related in this manner necessarily share discrete information about their state. The same is true of any tuple whose corresponding set of NSBs have empty intersection. We showed that NSBs adopt non-intersecting, onion-like structures, implying that very few measurements of network state are sufficient to determine the output label with a linear classifier. The "reason" $\VCind$ produces such pessimistic bounds is because in the worst case, every binary value of $\sigbar(x)$ is required to determine $\net(x)\geq 0$ up to linear classifier. We expect structure in the data to reduce capacity by excluding these worst cases. For linearly separable data, the the learned DNN classifier depends on \textit{no entry} of $\sigbar(x)$.


As a result, we have produced a powerful method for analyzing, interpreting, and improving DNNs. A deep network is a black box model for learning, but it need not be treated as such by those who study it. Our logical circuit formulation requires no assumptions and seems extremely promising for introspection and discussion of DNNs in many applications.



Whether our approach can be extended or adapted to other datasets is an pressing question for future research. An important and particularly difficult open question (precluding such an investigation presently) is the efficient determination of $\SigbarO$ (or even $\Sigbar$) analytically given the network weights. Such an algorithm seems prerequisite to bring deep logical circuit analysis to bear on datasets of higher dimension where we can no longer grid search. 

\section{Acknowledgements}
This work was supported by NSF under grants 1731754 and 1559997.

\newpage
\small
\bibliography{MonotonicityPaper}

\newpage


\onecolumn 
\section{Appendix}
\subsection{A Comparison of VC Dimension Bounds: Why Can We Get Away with Less?}
\label{sec:bartlett_compare}
The generalization bounds we presented can be extremely small, despite thematically similar to traditional VC dimension uniform bounds, denoted $\VCind$), which by contrast produce some of the largest bounds. 
That these measures differ by orders of magnitude (refer again to Table \ref{tab:comb_stat}) becomes more notable still once one realizes that "under the hood" they apply exactly the same theoretical machinery. Our workhorse theorem \ref{thm:Jerrum} can also be applied directly to derive tight bounds on $\VCind$ as in \citet{Bartlett2017b}. We reproduce it below for convenience.

\Jerrum*

In this section, we dissect the proof of the data independent VC dimension bound to understand more concretely what allows ours to be so much smaller. To rephrase, we want to understand through what mechanism the data-independent VC bound could potentially be improved if allowed additional assumptions on $\net$ of the form supported by our empirical observations of networks training on nice data. We will show that the notorious numerical bloat characteristic of $\VCind$ bounds can be  localized to a particular proof step, and that this step can be greatly accelerated when the set system of neuron state boundaries organized to minimize intersections. 

In addition to Figure \ref{fig:unused_capacity} and the surrounding discussion, we also provide in the Appendix a more complete catalogue of similar experiments in Figure \ref{fig:neuron_states}, which may be a useful reference for the following discussion.

We will summarize briefly Theorem $8$ in \citet{Bartlett2017b}. The theorem orders the neurons in the network so that every neuron in layer $l$ comes before every neuron in layer $l+1$, with the output neuron last, whose state is the label by convention.
Moving from beginning to end of this list of neurons, one asks how many additional polynomial
\footnote{With respect to the input, these are all linear inequalities. But, the weights parameterizing these linear functionals are composed by multiplication, so they are polynomial of degree bounded by the depth.} 
queries must be answered to determine the state of the ${i+1}^{\text{th}}$ neuron given the state of the first $1,\ldots,i$. 
The total resources in terms of the number of parameters $k$, polynomial inequalities $s$, and polynomial degree $d$ required to determine all the states is plugged into the VC bound in Theorem \ref{thm:Jerrum}.

Perhaps much of the gap is attributable to neuron state coupling. Beyond the fact that many neurons are simply always on or always off, those with a nonempty neuron state boundary (NSB) share \textit{a lot} of state information. 
A subset of $k$ neurons whose NSBs have empty intersection share state information, as not all $2^k$ states are possible. 
When they form parallel structures, as in Figure \ref{fig:neuron_states}, often the state of one will imply immediately the state of the other. 
Furthermore, we can infer by analyzing our Boolean formula (Eqn \ref{thm:sig_bdry}) that to determine the output label, we \textit{do not need to determine} the state of those neurons whose NSB does not intersect the DB. To  In our experiments, it is rare for even two such neuron boundaries to cross, and all seem repulsed from the decision boundary, implying that gradient descent has strong \textit{combinatorial} regularizing properties.

Note that strong coupling of neuron states in a single hidden layer does \textit{not} require linear dependence of the corresponding weight matrix rows: strong coupling is possible even for weight vectors in (linear) general position. For example, consider that for ReLU, the activations from the previous layer are always in the (closed) positive orthant, $\bbR_+^n$. General (linear) dependence means there is no nonzero linear combination of weight vectors reaching the origin,$\{0\}$,
which it will be instructive to consider as the dual of the entire space, $\{0\}=(\bbR^n)^*$. We propose that a more suitable notion of "general position" is instead the absence of nonzero linear combination of weights reaching the dual of some set containing the activation-image of the input. For simplicity, consider the positive orthant, $\bbR_+^n=(\bbR_+^n)^*$. Our new notion of dependence coincides with neuron state dependence in the next hidden layer. For example, suppose for $\gamma_i>0, \eta\in\bbR_+^n\setminus\{0\}$, we have $\gamma_1w_1+\gamma_2w_2-\gamma_3w_3-\gamma_4w_4=\eta$. Then because $\forall x\in\bbR_+^n$, we have $\eta^Tx>0$, we cannot simultaneously have $w_1^Tx,w_2^Tx<0$ while both $w_3^Tx,w_4^Tx>0$. Conversely if for all $x\in\bbR^n$, 
the states $w_1^Tx,w_2^Tx<0$ and $w_3^Tx>0$ imply $w_4^Tx>0$, then apply Farkas' lemma to get a positive linear combination of $-w_1,-w_2,w_3,w_4$ in $\bbR_+^n$.

These observations about dependencies between these neuron states illustrate nicely how the combinatorial capacity of networks trained on structured data diverges from the worst case theoretical analysis. Rather than handle one neuron at a time, we found it more useful to shift from a neuron-level to a network-level analysis by introducing network states $\sigbar(x)$, which we found to be significantly cleaner for interpretation and generalization bounds.

\begin{table*}[!htb]
  \caption{Measures of Combinatorial Capacity:}
  \label{tab:comb_stat}
  \centering
  \begin{tabular}{lllll}
    \toprule
    \multicolumn{2}{c}{} & \multicolumn{3}{c}{$\VCbool(\net)$($|\SigbarO|)[|\Sigbar|$]}\\
    \cmidrule(r){3-5}
    Architecture & $\VCind(\text{Arch})$   & {\Datai}     & {\Dataii} & {\Dataiii} \\
    \midrule
    {\Archi} & 8,400   &  18(1)[36]  & 380(10)[121]   & 1870(38)[247]      \\
    {\Archii} & 101,000 &  74(2)[80]  & 252(8)[166]   & 806(19)[507]    \\
    {\Archiii} & 710,000 &  74(2)[167]  & 335(6)[622]   & 22,000(144)[2884]   \\
    \bottomrule
  \end{tabular}
\end{table*}
\subsection{Further Discussion on Simple Observations}

\def\insertscale{.05}
\def\inslocx{05}
\def\inslocy{60}


\begin{figure}[!htb]
\includegraphics[width=\textwidth]{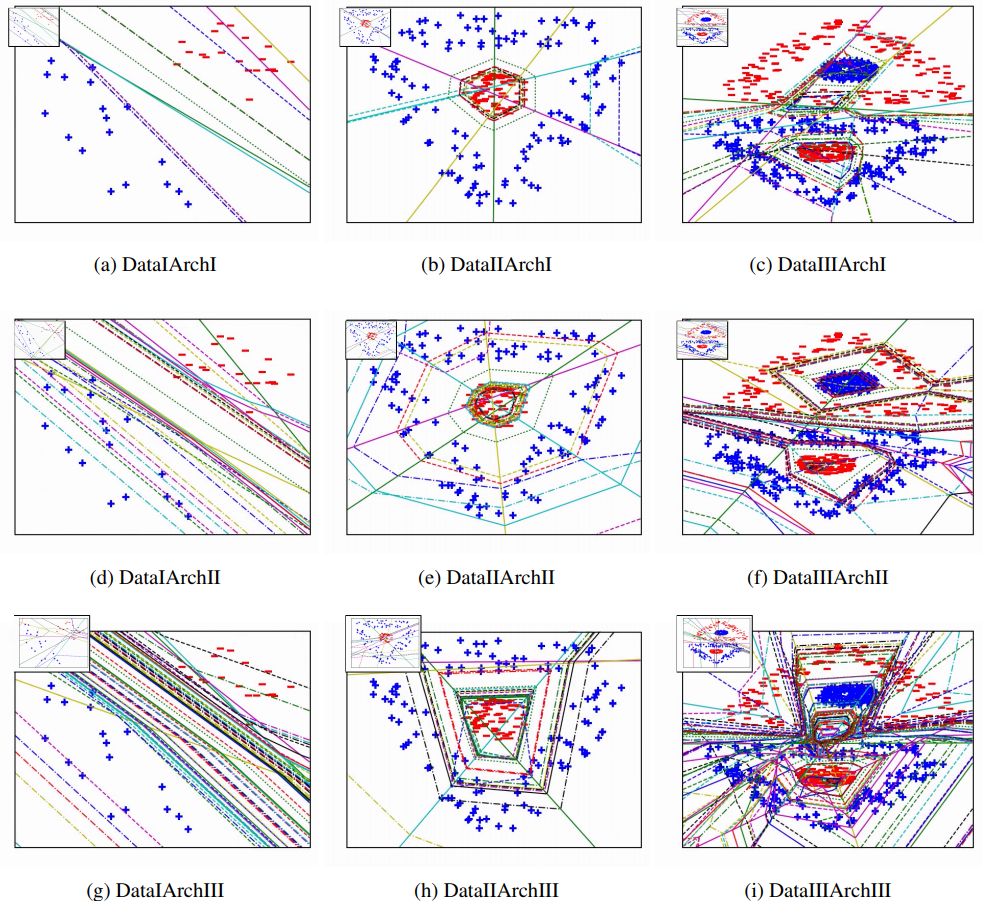}
\caption{All neuron state boundaries (NSBs) of $9$ networks after training (large) and near the start of training (each top left insert) with varying data and architecture complexity. Animations of this process for all $9$ experiments are available in the supplemental material. For {\Archi} (top row), the $4$ line styles in order of increasing dotted-ness, solid, dashed, dash-dot, and dotted, correspond to increasing layer numbers. For all networks, we reserve the most dotted line for the decision boundary. We make the following observations: 
 For fixed data, increasing the architecture size increases the number of linear regions, $|\Sigbar|$, but not necessarily the number of linear decision boundary pieces, $|\SigbarO|$, making it a more plausible candidate to relate to generalization error, which is also architecture size invariant. 
For fixed architecture, varying the data complexity controls the number of boundary pieces, which seem to be \textit{around} as few as needed to separate the regions of positive and negative density.
There is the appearance of a sort of "repulsive force", most readily apparent in the {\Archii} figures, between the decision boundary and the other NSBs. Recall that the decision boundary only "bends" when it intersects other NSBs. Such a force would then have a regularizing effect that minimizes the number of boundary pieces.
Regions where individual neurons are active tend to nest hierarchically with somewhat parallel boundaries, producing fewer possible network states.
}
\label{fig:neuron_states}
\end{figure}

\begin{figure}
  \centering
         \begin{subfigure}[b]{\linewidth}
                \centering
                \includegraphics[width=\textwidth]{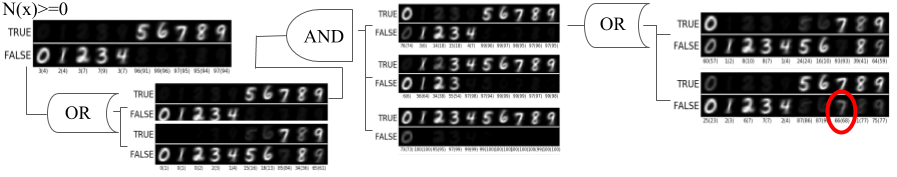}
                \caption{A Strange Way to Learn $7$s}
         \end{subfigure}
  \caption{This is a companion figure to Figure \ref{fig:mnist}, which contains generic instructions on how to read this type of diagram. Depicted is a selected subset of the logical circuit of a network trained on MNIST. This circuit is particularly large, and all depicted classifiers are Boolean combinations of other classifiers that are not shown. Similarly to the redundancy of Figure \ref{fig:bool_equiv_D2A3}, many of the omitted are similar to those shown. In contrast to the very sensible organization of circuit Fig\ref{fig:mnist-a}, the point of this figure is to show that these intermediate logical steps, though interpretable, can also be counter-intuitive, even circuitous seeming. 
  We proceed from right to left. Already two logical operations before the terminal network output, one of the classifiers (row $2$ far right) has \textit{nearly} separated the data (images of training data shown here), except that it has difficulty with many $7$s (red circle). Curiously, this is amended by combining using OR with a classifier that labels \textit{both} $7$ and $0$ as True. This is curious. It is not clear why including $0$ should be necessary or useful when including $7$. To remedy the new problem of returning True for $0$, it uses an AND junction to combine with another intermediate classifier that outputs False for $0,1,2,3$. .}
  \label{fig:mnist-b}
\end{figure}

\begin{figure}
\def\nsbwid{0.3\textwidth}
\def\nsbht{2.8cm}
  \centering
    \begin{tabular}[t]{ccc}
        \begin{tabular}{c}
         \begin{subfigure}[b]{\nsbwid}
                \centering
                \includegraphics[height=\nsbht{}]{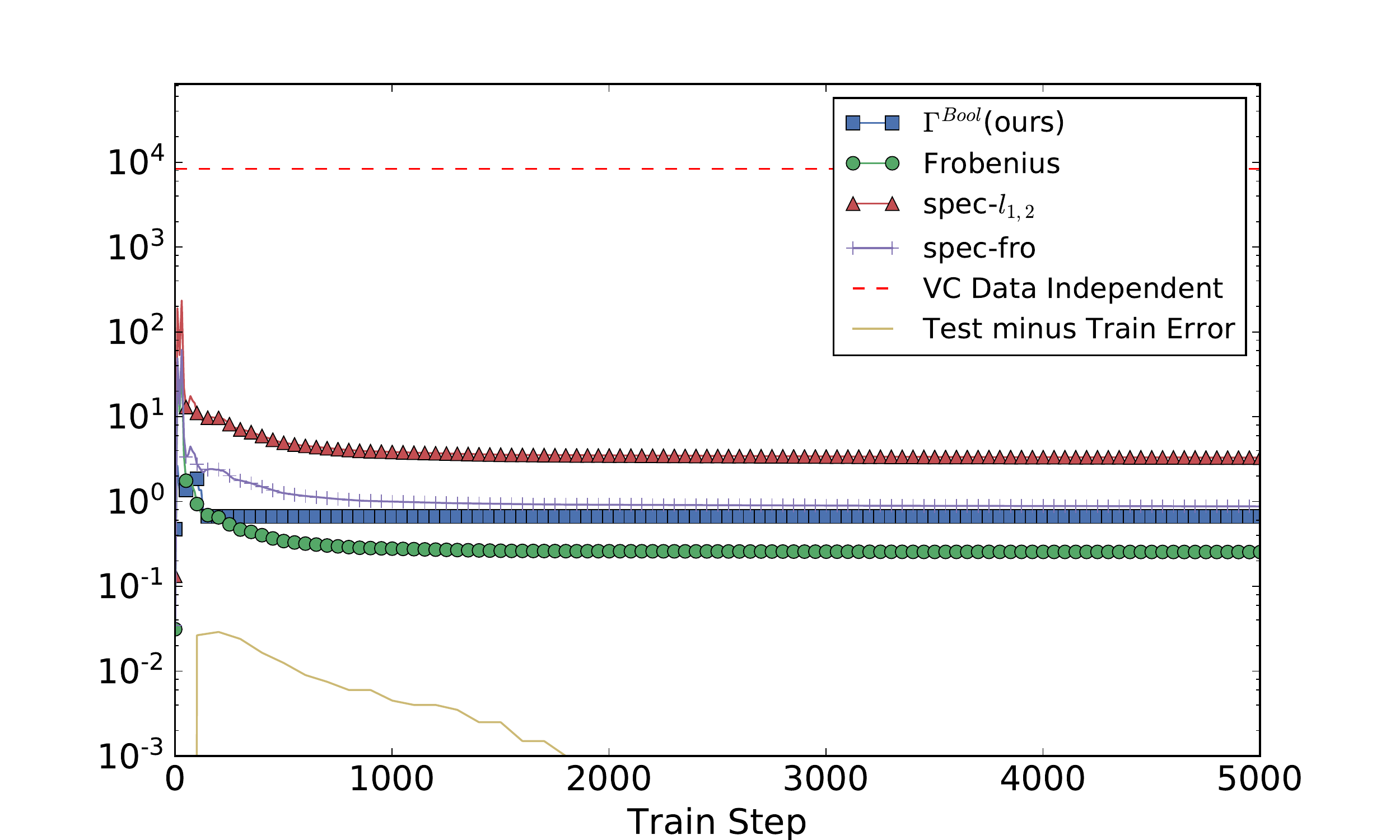}
                \caption{\Datai{}\Archi{}}
         \end{subfigure}\\
         \begin{subfigure}[b]{\nsbwid}
                \centering
                \includegraphics[height=\nsbht{}]{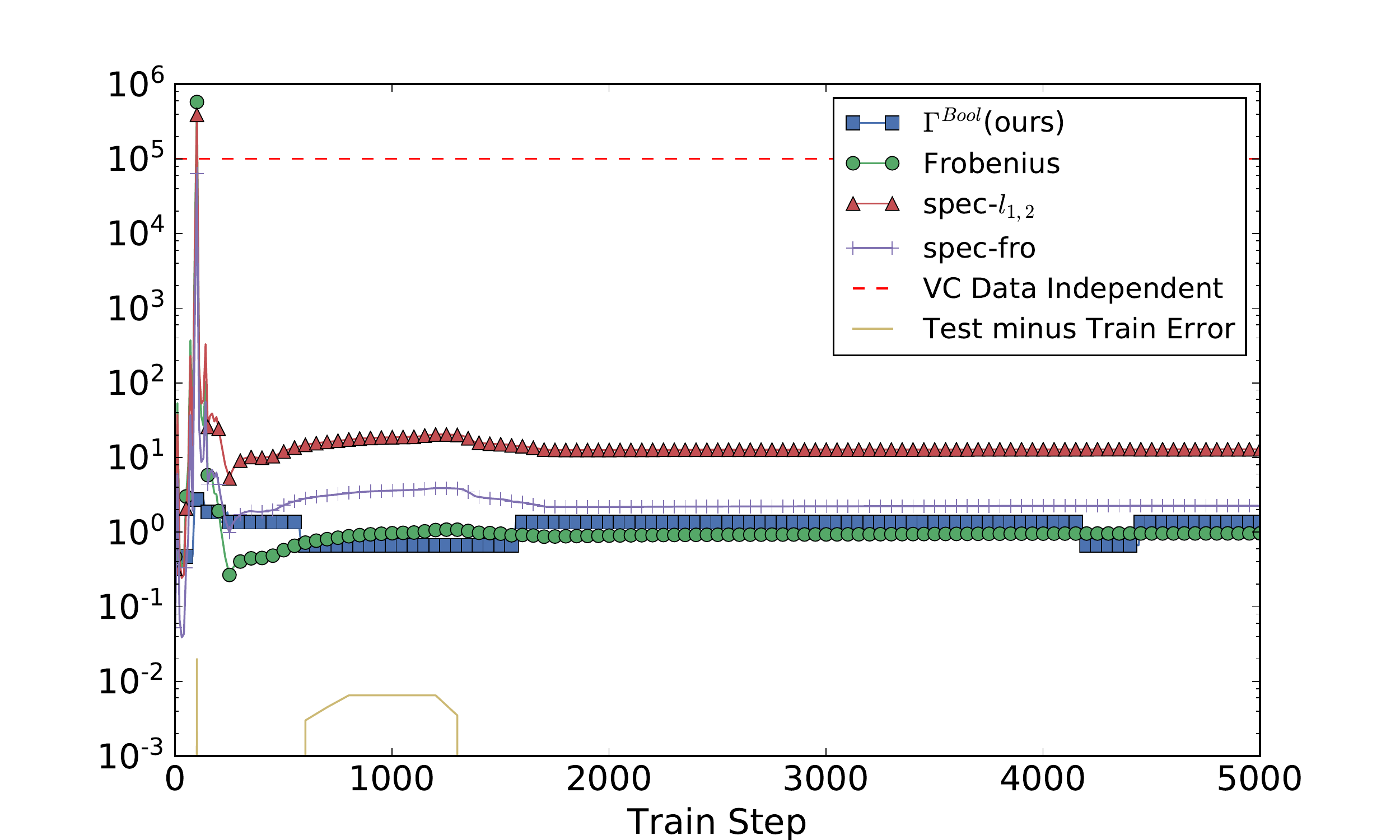}
                \caption{\Datai{}\Archii{}}
         \end{subfigure}\\
         \begin{subfigure}[b]{\nsbwid}
                \centering
                \includegraphics[height=\nsbht{}]{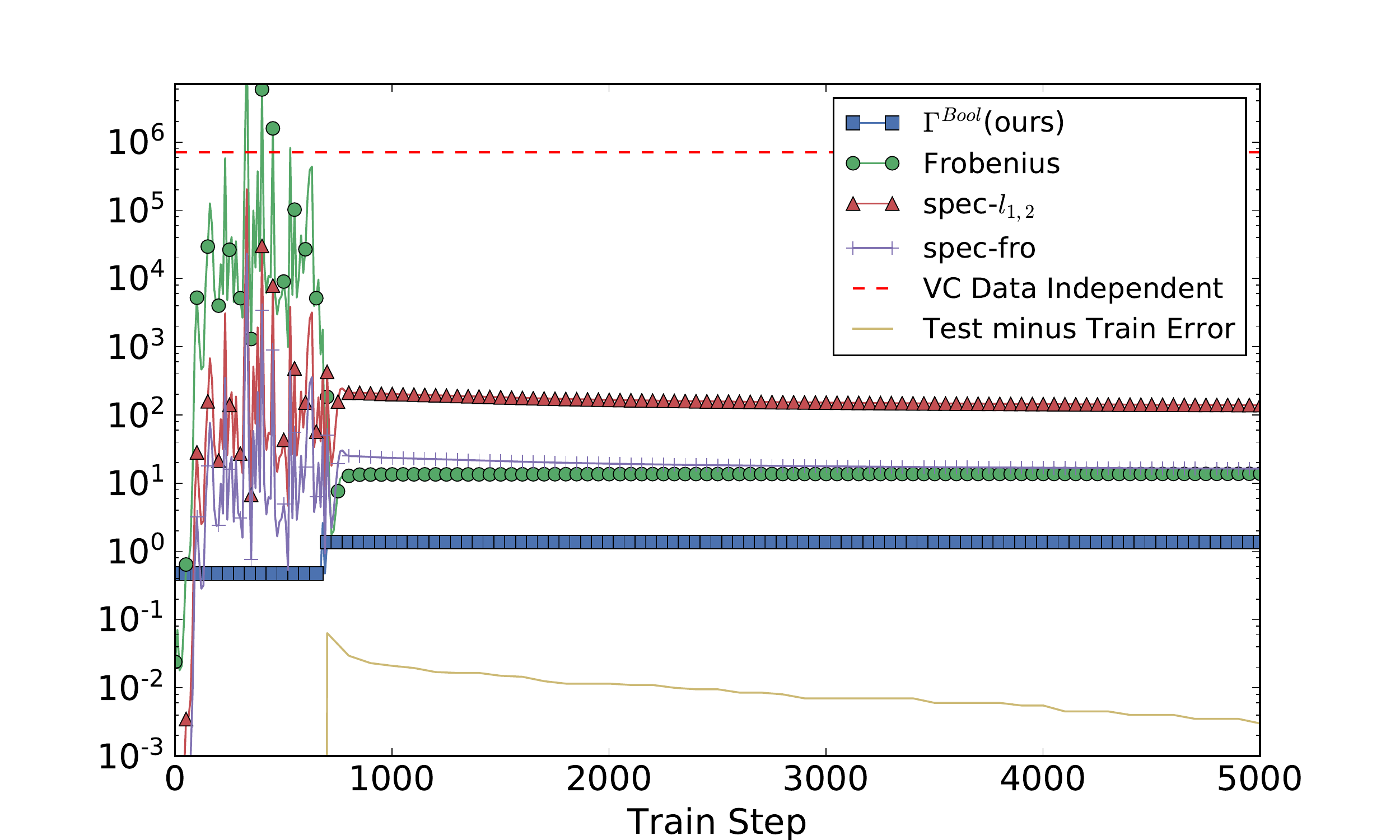}
                \caption{\Datai{}\Archiii{}}
         \end{subfigure}
        \end{tabular}
    &
        \begin{tabular}{c}
         \begin{subfigure}[b]{\nsbwid}
                \centering
                \includegraphics[height=\nsbht{}]{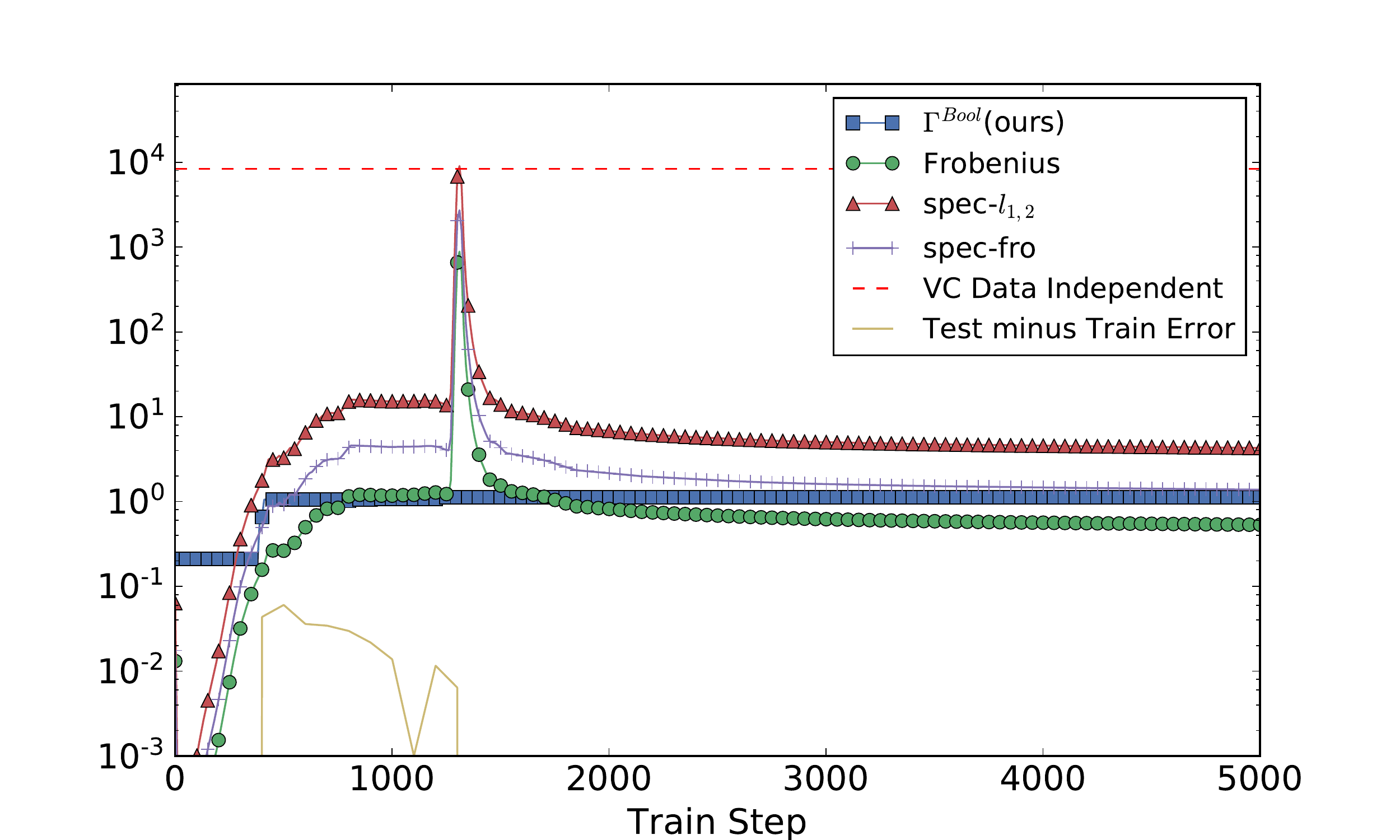}
                \caption{\Dataii{}\Archi{}}
         \end{subfigure}\\
         \begin{subfigure}[b]{\nsbwid}
                \centering
                \includegraphics[height=\nsbht{}]{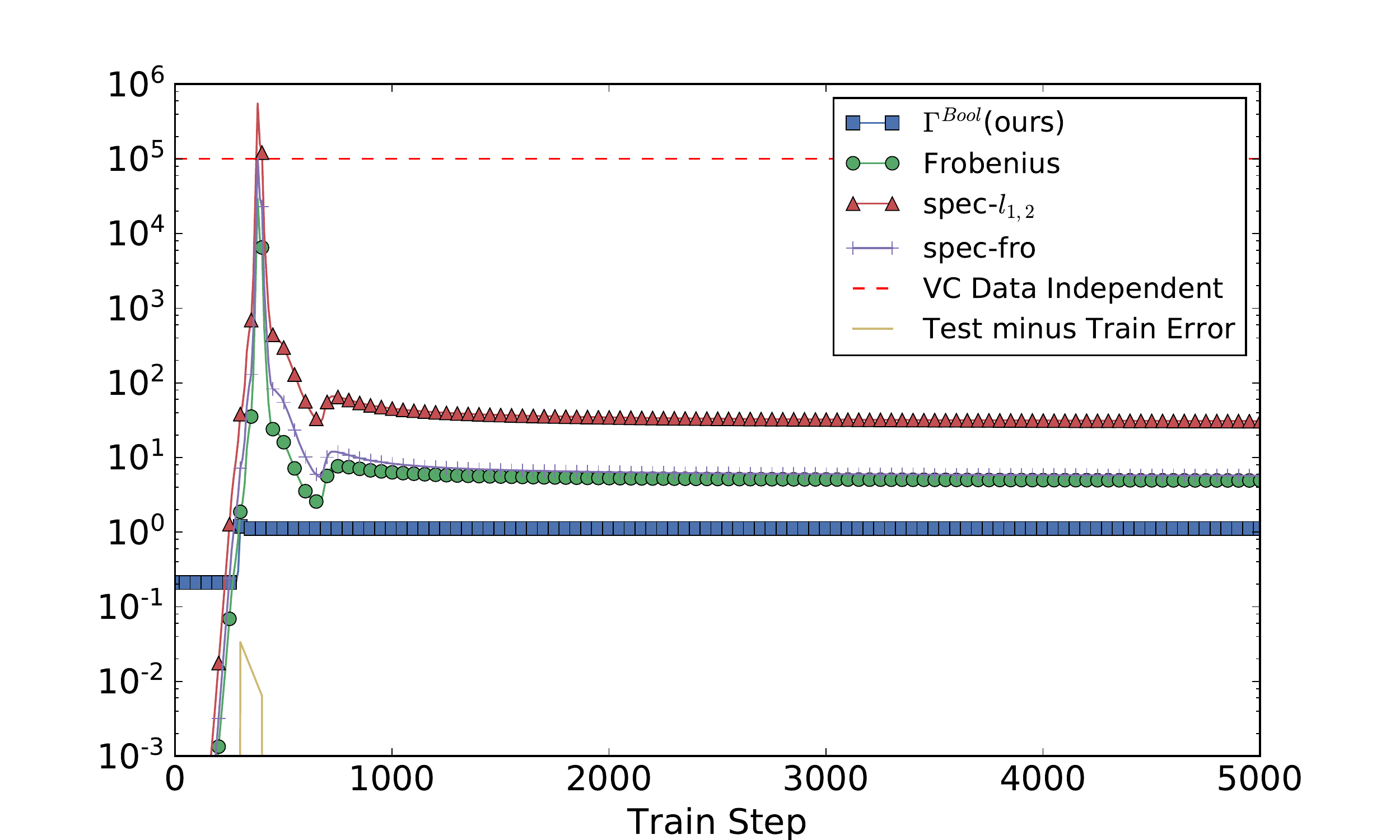}
                \caption{\Dataii{}\Archii{}}
         \end{subfigure}\\
         \begin{subfigure}[b]{\nsbwid}
                \centering
                \includegraphics[height=\nsbht{}]{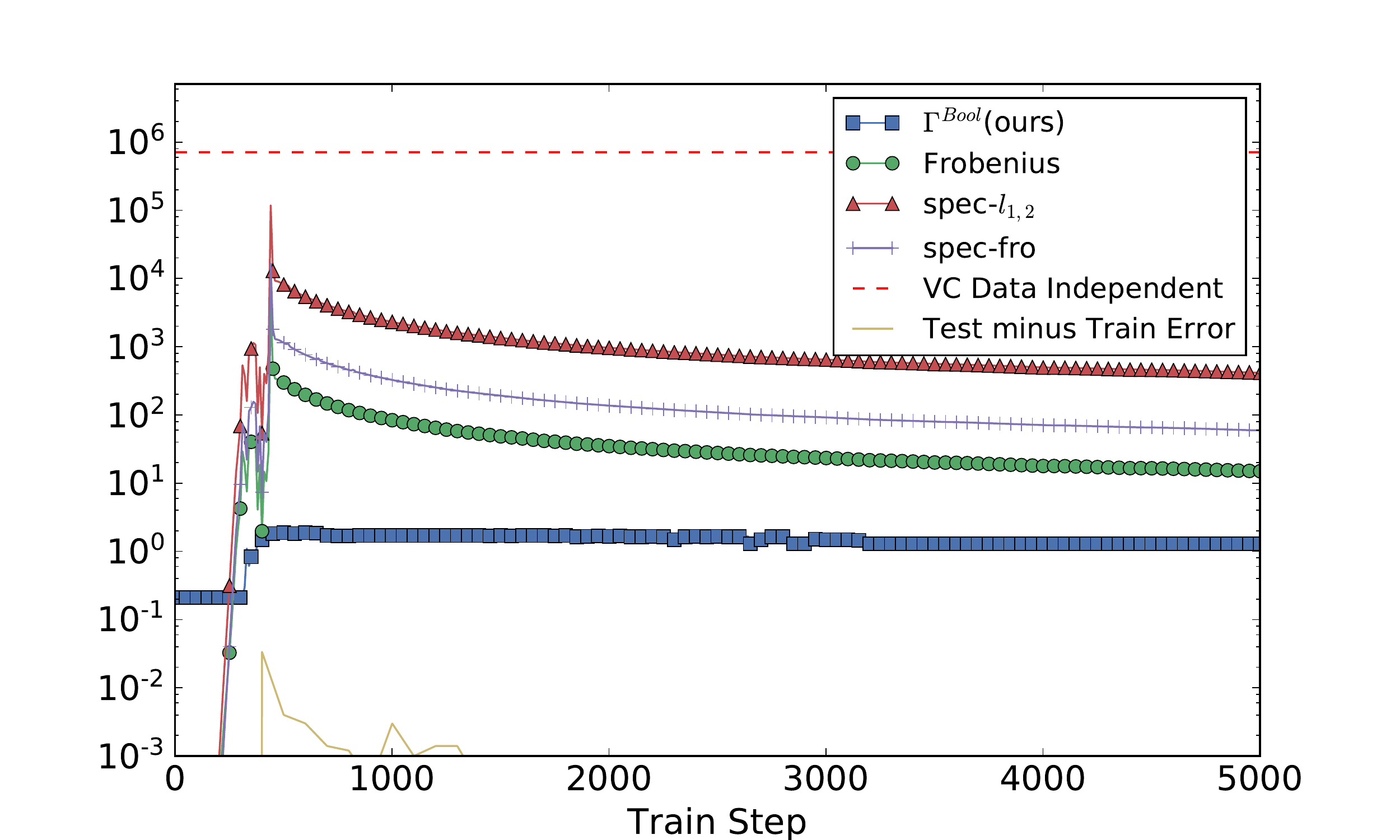}
                \caption{\Dataii{}\Archiii{}}
         \end{subfigure}
        \end{tabular}
&
        \begin{tabular}{c}
         \begin{subfigure}[b]{\nsbwid}
                \centering
                \includegraphics[height=\nsbht{}]{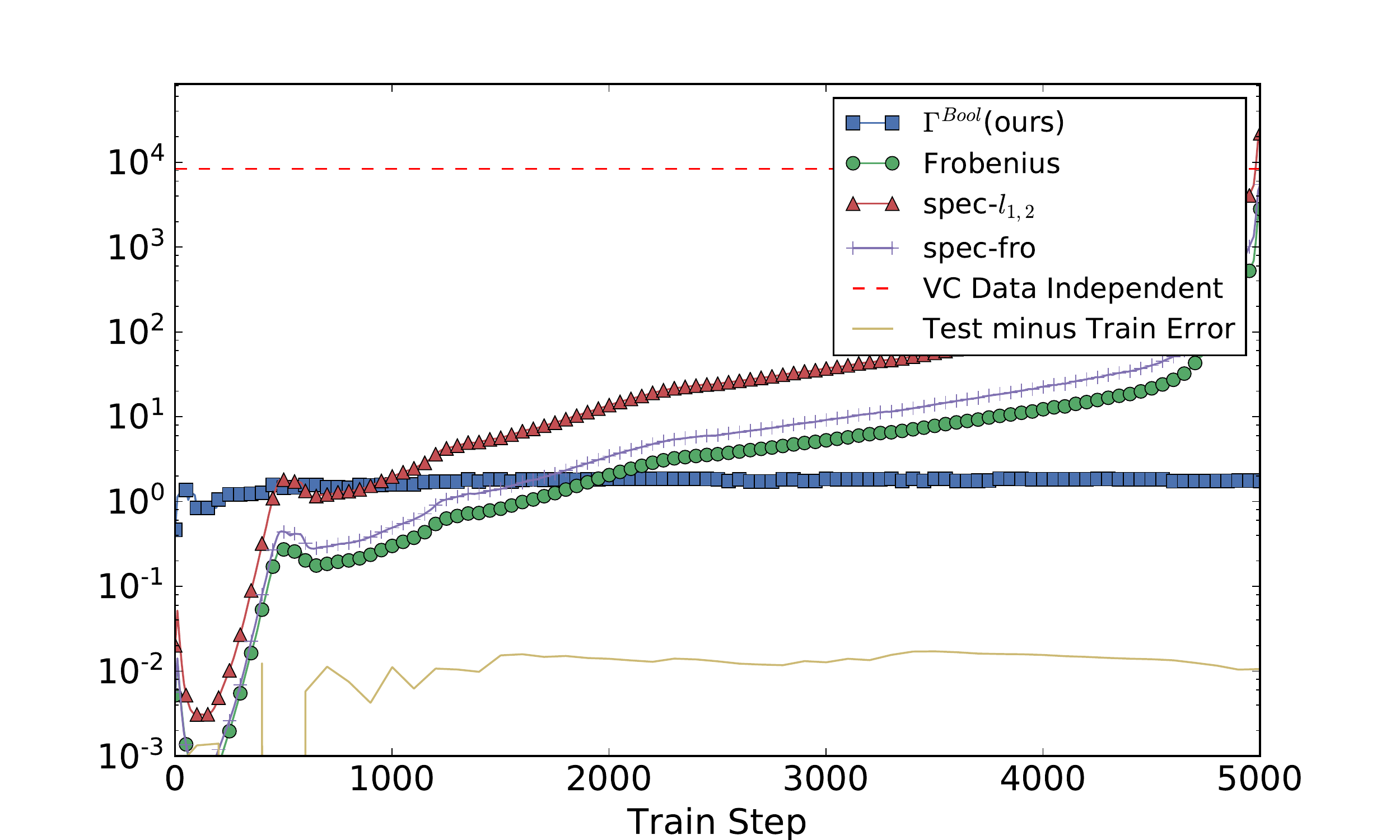}
                \caption{\Dataiii{}\Archi{}}                
         \end{subfigure}\\
         \begin{subfigure}{\nsbwid}
                \centering
                \includegraphics[height=\nsbht{}]{figures/FigGen/FigBounds_D3A1.pdf}
                \caption{\Dataiii{}\Archi{}}                
         \end{subfigure}\\
         \begin{subfigure}{\nsbwid}
                \centering
                \includegraphics[height=\nsbht{}]{figures/FigGen/FigBounds_D3A1.pdf}
                \caption{\Dataiii{}\Archi{}}
         \end{subfigure}
        \end{tabular}\\
    \end{tabular}
  \caption{Generalization bounds for all experiments. See Figure \ref{fig:gen_bound} for additional details.
    }
\label{fig:all_bounds}
\end{figure}

\subsection{Experimental Conditions}\label{sec:exp_details}

This training process can be seen as a function mapping a training set, network architecture pair to the DNN classifier and thus also to the generalization error we aim to study. As such, we are interested in observing how the complexity of our learned classifier depends on the "architecture complexity" and "data complexity", which we treat as independent variables. 
For this purpose, we designate $3$ different network architectures, {\Archi}, {\Archii}, and {\Archiii}, of increasing size and three different datasets, {\Datai}, {\Dataii}, and {\Dataiii}, of increasing "complexity" (both pictured in Appendix Figure \ref{fig:arch_diagrams},\ref{fig:data_diagrams}).
Together, these comprise $9$ experiments total. Rather than define explicitly what makes a dataset complex, to justify {\Datai}, {\Dataii}, and {\Dataiii} are of "increasing complexity", we simply note that these datasets are nested by construction, and, as a result, so too are the sets of classifiers that fit the data \footnote{There are many ways to define data complexity, but it's not clear which one should apply here. The ability to convincingly vary the data complexity in spite of this is in fact a key advantage of using synthetic data.}.

The three architecture sizes were chosen as our best guesses for the widest range of sizes our grid search algorithm would support comfortably. We do not recall changing them thereafter. The datasets were the the simplest interesting trio with the nesting property. The specific scaling and shifting configuration hard-coded into \Datai{},\Dataii{},\Dataiii{} was simply the first one we found (after not much search) that allowed the shallowest network to achieve $0$ training error on the most difficult dataset. It is somewhat important for the data to be centered. Also, using Adam rather than gradient descent made this much easier, so we decided to standardize all our experiments to Adam(beta1=0.9,beta2=0.999) (the Tensorflow default configuration).

All experiments had a learning rate of $0.0005$, biases initialized to $0$, multiplicative weights initialized with samples from a mean, $0$, standard deviation, $0.05$, truncated normal distribution. These were originally set before the lifetime of this project in some code we re-purposed. We don't recall ever changing these thereafter.

Unlike the experiments in the rest of the paper, the deep logical circuits displayed in the MNIST Figures \ref{fig:mnist},\ref{fig:mnist-b}, are not identical from run to run. Instead, it seems almost every run gives a circuit that is at least slightly different. Initially, our intention was to train networks to label just a few digits, (instead of separating $0-4$ from $5-9$, but these circuits all turned out to be too trivial to be interesting. Our experimental design was to vary the architecture and the number of training samples, $m$. Our implementation in network\_tree\_decomposition.py not optimized, so we only had time for $21$ of these runs. Most runs were interesting, so the real criteria was whether they could fit on a page cleanly.

In order to determine $\SigbarO$ for the MNIST experiments, we use a trick. Instead of performing a grid search over the input, which is $784$ dimensional, we learn and additional $4$ to $6$ width \textit{linear} layer before the first hidden layer of each architecture. We then do grid search over the first linear layer and compose with the projection map in post-processing.

We now cover some specific details of the circuits that were shown. Figure \ref{fig:mnist-a} is $0.98$($0.94$) training(validation) accuracy {\Archi} network with a $6$ neuron linear layer, trained with $m=50$k samples. The circuit in Figure \ref{fig:mnist-b} has the same settings and has $0.96(0.93)$ training(test) accuracy.

The overfit circuit with $1.0(0.78)$ train(test) accuracy in Figure \ref{fig:mnist-c} has a $4$ neuron linear layer, uses {\Archiii}, and has only $1$k training points. The prosthetic network that we trained, also used {\Archiii} and used a subset of the \textit{same} training data. Since we trained it to label $4$ as False and $5-9$ as True, we had a class imbalance issue, so we subsampled $30$\% of the True labels. Therefore, the "prosthetic network" was trained with only $m=250$ training samples.

All experiments were run comfortably on a TitanX GPU.

\begin{figure}
  \centering
   \includegraphics[width=0.8\linewidth]{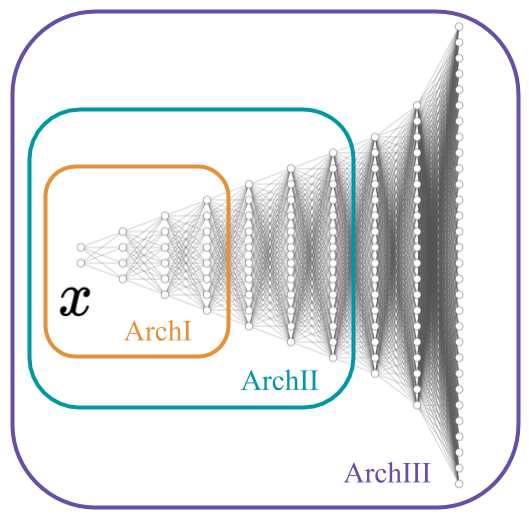}
  \caption{This figure defines the network architectures, {\Archi}, {\Archii}, and {\Archiii}, used for binary classification experiments in this work. The two leftmost neurons represent 2-dimensional input data, while the remainder are hidden. Lines correspond to multiplicative weights between neurons. Biases are used in experiments but not shown in this figure. There are additional multiplicative weights not shown that connect the final hidden (equiv. rightmost) layer of each network to a scalar output. The architectures are nested, e.g., the structure of the first $3$ hidden layers is constant across all experiments. The first $6$ hidden layers have the same structure for {\Archii} and {\Archiii}.}
  \label{fig:arch_diagrams}
\end{figure}

\begin{figure}
  \centering
  \includegraphics[width=0.8\linewidth]{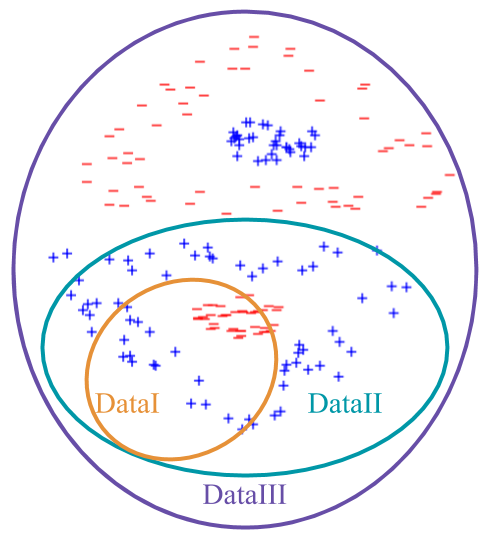}
  \caption{Displayed is one possible sampling of the three data distributions used for experiments in this paper. Blue plus[red minus] signs correspond to training data from the positive[negative] class. We designed our datasets to explore the effect of classification difficulty/complexity on neural network VC dimension. Although we do not yet understand exactly what properties make a dataset "more difficult", we can reasonably expect datasets ordered by inclusion to also be ordered in complexity, e.g., every hypothesis that correctly classifies {\Dataii} also correctly classifies {\Datai}. It should be noted that this inclusion ordering displayed in this figure is up to affine transformation. For example, the figure should be interpreted to mean that there is some affine transformation so that transformed samples from {\Dataii} follow the same distribution of samples from {\Dataiii} that lie lower than average. In fact, {\Datai}, {\Dataii}, and {\Dataiii} all have roughly the same center of mass in experiments.
  }
  \label{fig:data_diagrams}
\end{figure}

\begin{table}
  \caption{Architecture Properties}
  \label{tab:arch_prop}
  \centering
  \begin{tabular}{cccc}
    Architecture     & Depth (\depth) & Parameters & VCdim \\
    \midrule
    {\Archi}  &  3  &  107   & 8376     \\
    {\Archii}  &  6  &  517   & 101110   \\
    {\Archiii}  &  9  &  1743  & 709558   \\
    \bottomrule
  \end{tabular}
\end{table}

\clearpage
\subsection{Experimental Support for Theoretical Results}\label{sec:exp-support}
In this section we use the included file "network\_tree\_decomposition.py", which implements Algorithm \ref{alg:op_tree}, to show experimental support for Theorems \ref{thm:hierarchical} and \ref{thm:sig_bdry}. Specifically we show that Equation \ref{eqn:EquivSigO} holds everywhere in the input space. Not mentioned in the theory section, we also show that the same indexing trick can be applied to Eqn $1$ of Theorem \ref{thm:hierarchical}: if one replaces $\SigbarO$ with $\Sigbar$ then one recovers a hierarchical MinMax$\ldots$MinMax formulation that is numerically equal to $\net(\x)$.

\begin{figure}
\def\nsbwid{0.3\textwidth}
\def\nsbht{2.8cm}
  \centering
         \begin{subfigure}[b]{\nsbwid}
                \centering
                \includegraphics[height=\nsbht{}]{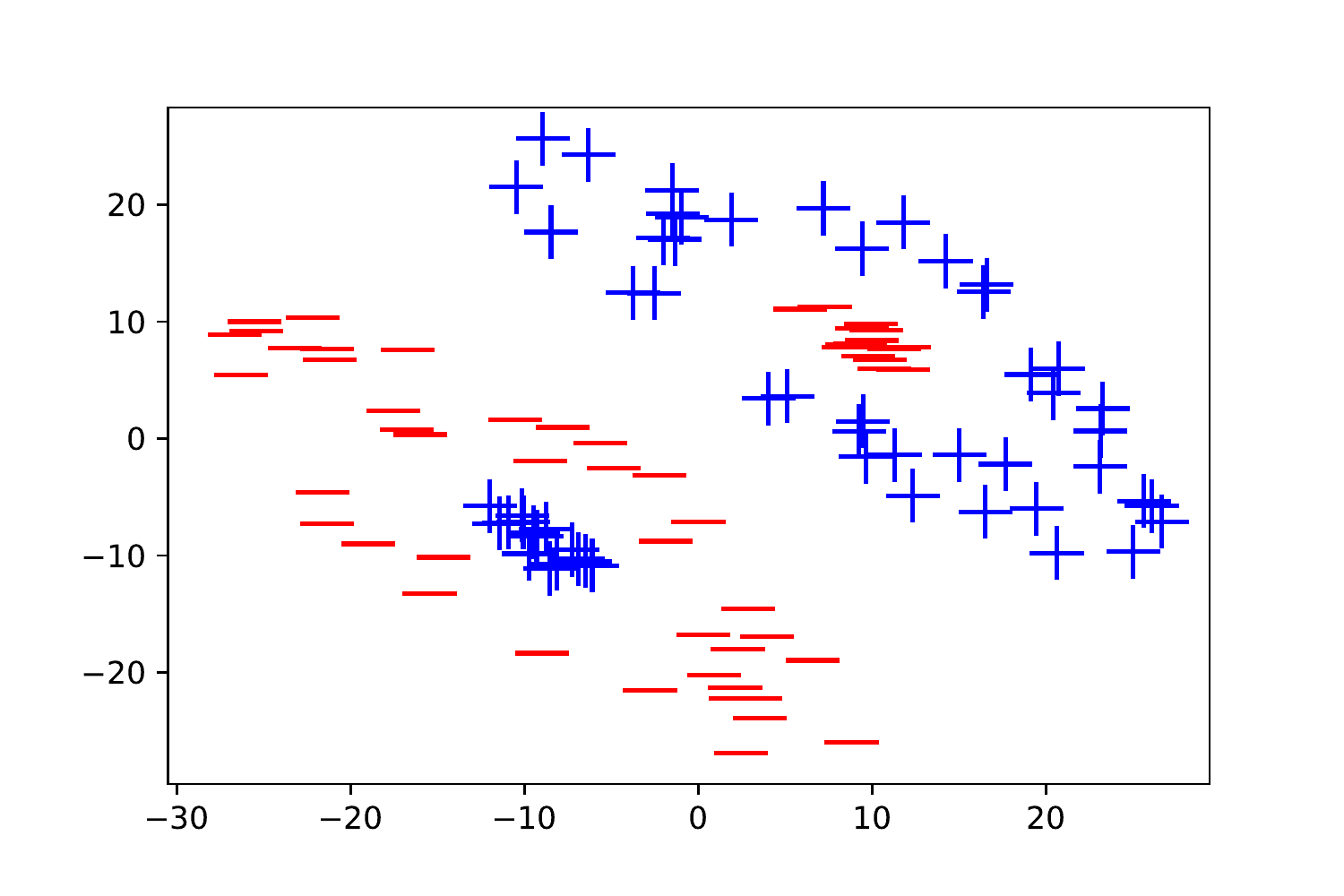}
                \caption{The training data}
                \label{fig:es-a}
         \end{subfigure}
         \begin{subfigure}[b]{\nsbwid}
                \centering
                \includegraphics[height=\nsbht{}]{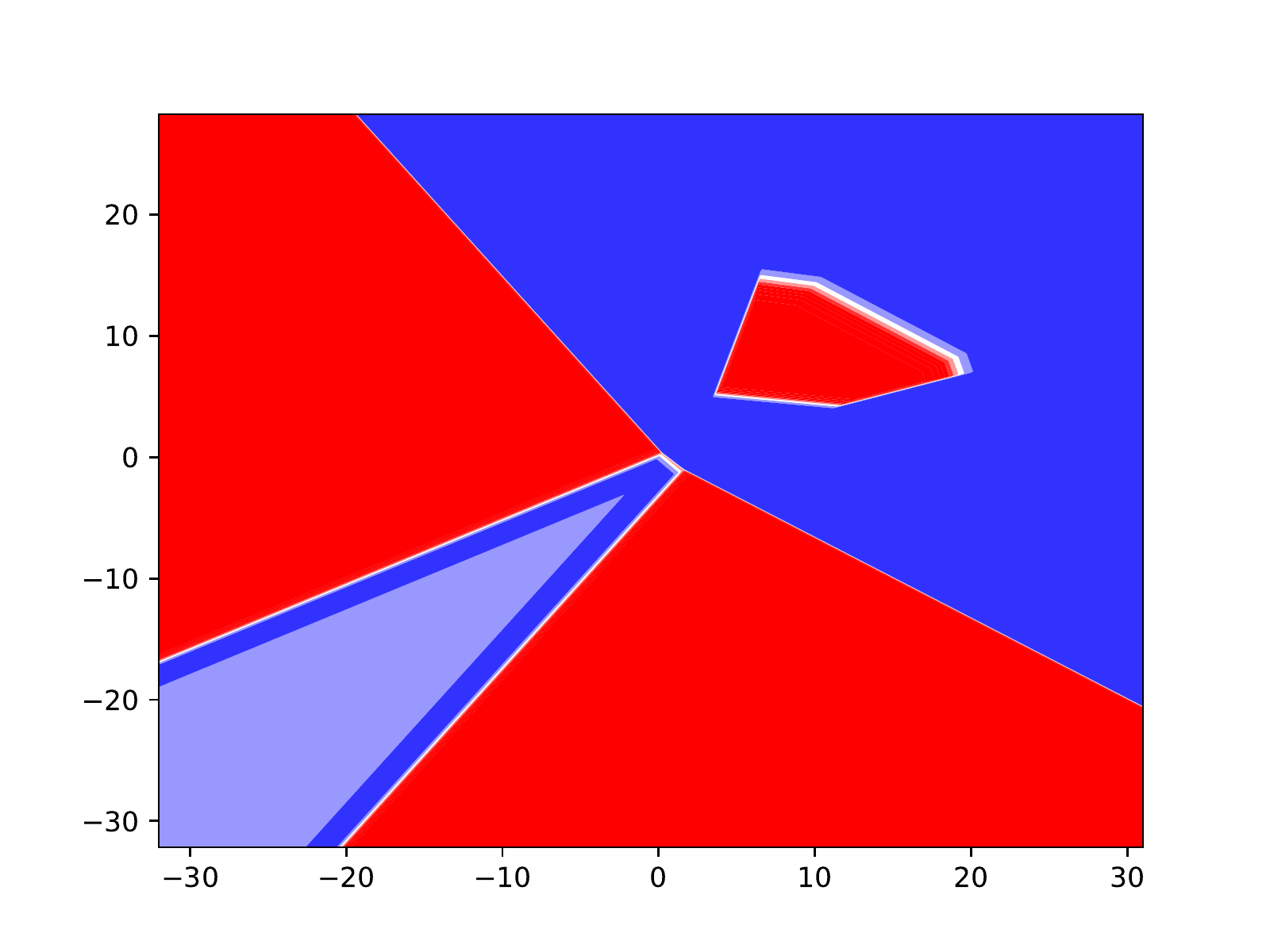}
                \caption{Learned Classification}
                \label{fig:es-b}
         \end{subfigure}
         \begin{subfigure}[b]{\nsbwid}
                \centering
                \includegraphics[height=\nsbht{}]{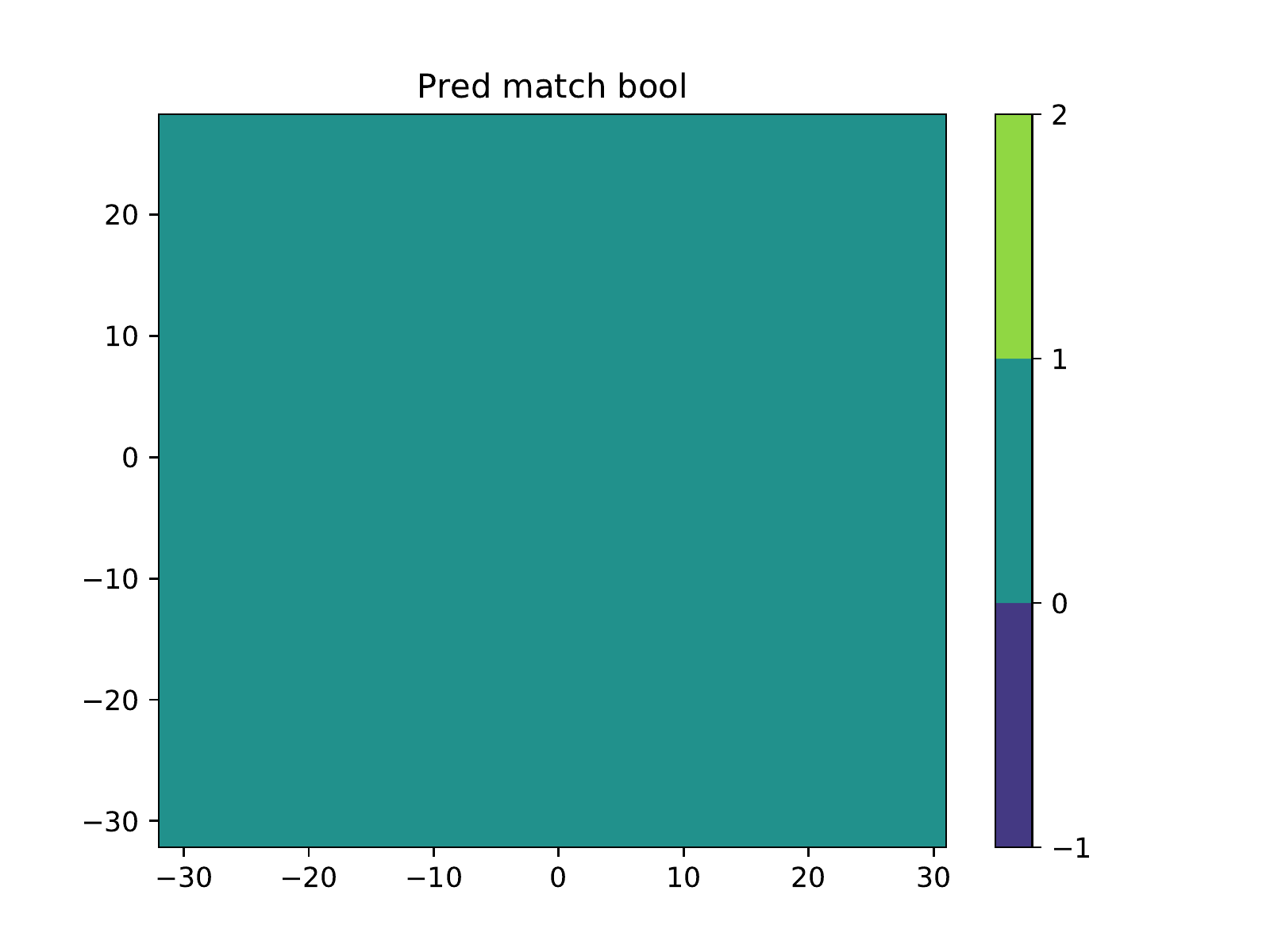}
                \caption{Validate $\net$pred$==$Bool }
                \label{fig:es-c}
         \end{subfigure}
         
         \begin{subfigure}[b]{\textwidth}
                \centering
                \includegraphics[width=\linewidth]{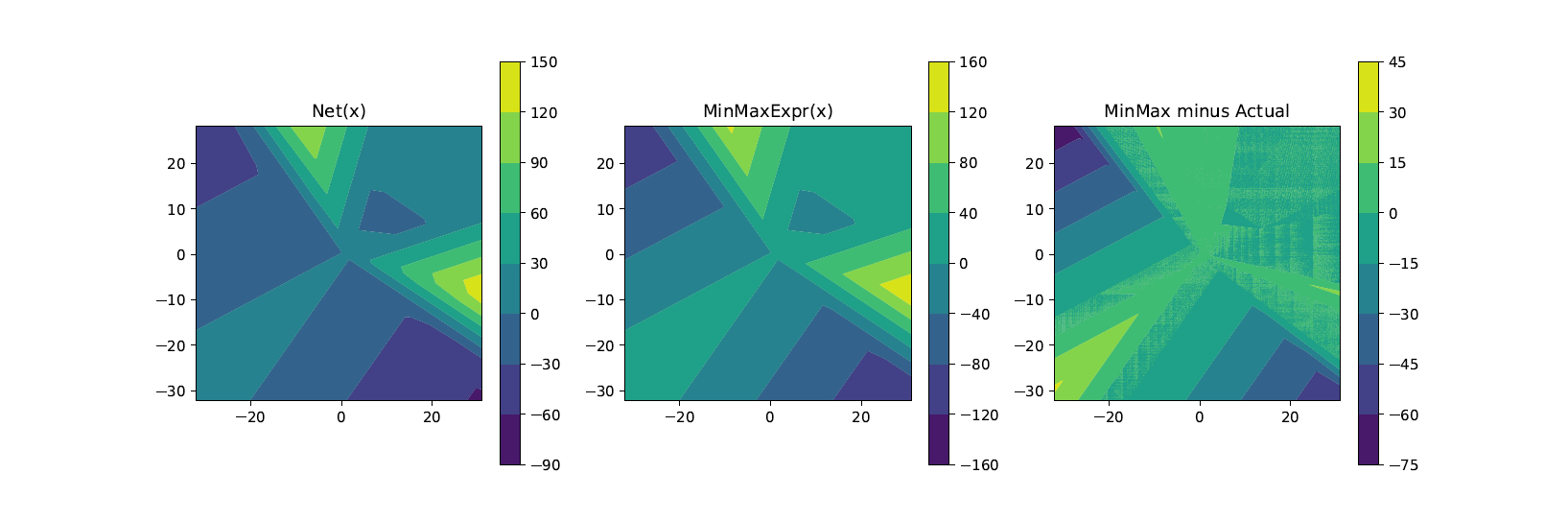}
                \caption{Algorithm \ref{alg:op_tree} (Mode$=$Logical) Output Comparison }
                \label{fig:es-d}
         \end{subfigure}
         
         \begin{subfigure}[b]{\textwidth}
                \centering
                \includegraphics[width=\linewidth]{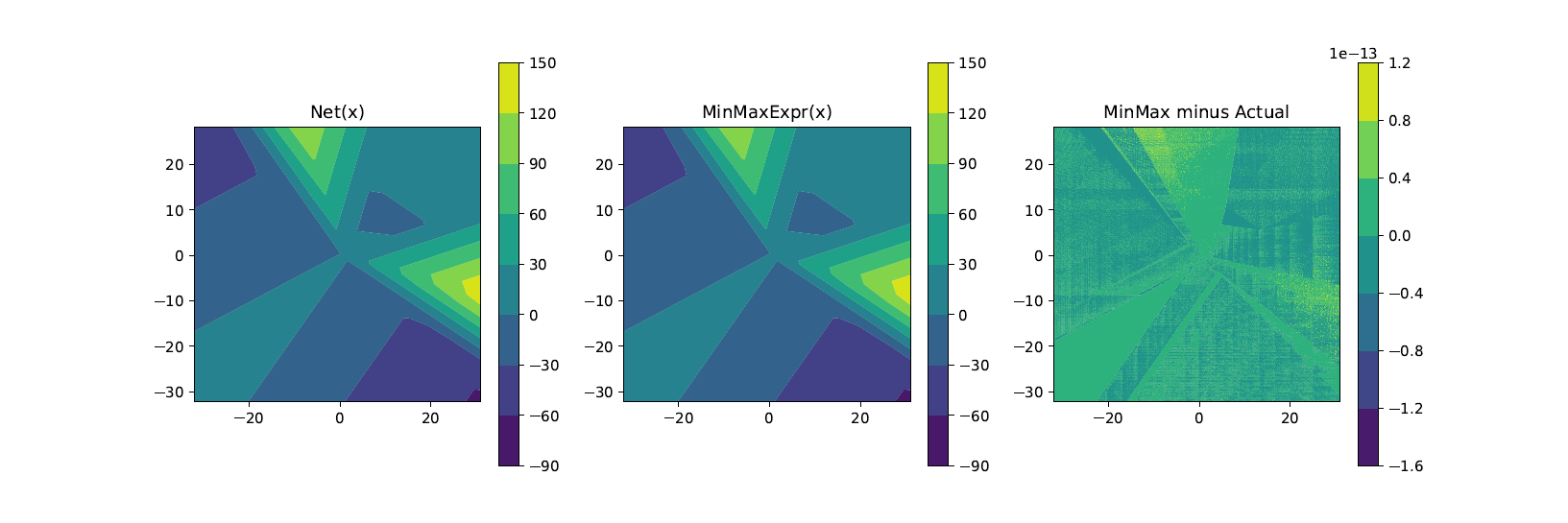}
                \caption{Algorithm \ref{alg:op_tree} (Mode$=$Numeric) Output Comparison }
                \label{fig:es-e}
         \end{subfigure}
  \caption{Experimental readout from network\_tree\_decomposition.py, which implements Algorithm \ref{alg:op_tree}, confirming the validity of our theoretical claims. Pictured is the classification boundary \ref{fig:es-b} of a DNN trained on \ref{fig:es-a}. The second[third] row (Fig. \ref{fig:es-d} [Fig. \ref{fig:es-e}]) plots the numeric value of the MinMax tree when indexed over $\SigbarO$[$\Sigbar$] corresponding to mode$=$Logical [mode$=$Numeric] in Algorithm \ref{alg:op_tree}. From left to right, rows $2$ and $3$ depect the network output, MinMax tree output, and their difference. For mode$=$Numeric, their difference is within machine precision of $0$ (row $3$ column $3$). For the second row, mode$=$Logical uses the same MinMax formulation, but indexes over $\SigbarO$ instead (It is Thm \ref{thm:sig_bdry} but with MinMax instead of $\wedge\vee$). We can see the numeric relation to $\net(\x)$ is lost (row $2$ column $3$). However, the \textit{sign} of this output still agrees with $\net(\x)$ everywhere. One can check this visually by comparing row $2$ columns $1$ and $2$. Or, one can refer to Fig. \ref{fig:es-c}, where we plot $1$ everywhere the two are equivalent. Because we see a constant image that is $1$ everywhere, we can infer our predictions are the same for all labels.
    }
\label{fig:exp_supp}
\end{figure}


\subsection{Algorithms: Definitions and Pseudocode}\label{app:algorithms}

\begin{algorithm}
  \DontPrintSemicolon
  \SetKwFunction{ScanSig}{BdryStates}
  \SetKwFunction{MinDesc}{MinimalDescrip}
  \SetKwFunction{GenBnd}{Thm\ref{thm:Jerrum}Bound}
  \SetKwProg{Fn}{Function}{:}{}
  \Fn{\ScanSig{$x\mapsto\sigbar(x),\net(x)$, $\mathcal{X}^{\text{compact}}$, $\delta$}}{
        Grid=MakeGrid($\mathcal{X}$,$\mathcal{X}^{\text{compact}}$,spacing=$\delta$)\;
        $\SigbarM$,$\SigbarP$=$\varnothing$,$\varnothing$\;
          \For{$x$ in Grid}{
              \uIf{$\net(x)\geq 0$}{
                $\SigbarP\gets\SigbarP\cup\{x\}$\;
                }
              \uElseIf{$\net(x)< 0$}{
                $\SigbarM\gets\SigbarM\cup\{x\}$\;
                }
              }
        $\SigbarO=\SigbarP\cap\SigbarM$\;
        \KwRet $\SigbarO$\;
  }
  \;
  \SetKwProg{Pn}{Function}{:}{\KwRet}
  \Pn{\MinDesc{$\widthbar$, $\SigbarO$}}{
        $s=|\SigbarO|^2$\;
        $\rankl{d+1}=1$\;
        \For{$l\in d,\ldots,1$}{
        $t_l=rank(\SigbarOl{l}$)\;
        $m_l=|\{i\in\width{l}|\exists \sigma\in\SigbarOl{l}$ with $\sigma_i\not=0\}$\;
        $\rankl{l}=\min\{m_l,t_l\rankl{l+1}\}$\;
        }
        $k=\sum_{l=0}^d (r_{l}-1)r_{l+1}$\;
        $\rankbar=\defrankbar$\%effective widths\;
        $d=|\{l | r_l>1\}|$\%effective depth\;
        \KwRet $k,s,d,\rankbar$\;
  }
  \SetKwProg{Gn}{Function}{:}{\KwRet}
  \Gn{\GenBnd{$k,s,d,m$}}{
        $VC=2k \log_2( 8e s d)$\;
        \KwRet $\sqrt{VC / m}$\;
  }
\caption{Generalization Bound Calculation}
\label{alg:gen_bound}
\end{algorithm}







\LinesNumbered
\SetKwInput{KwLet}{Let}
\begin{algorithm}[htbp]
\def\Affw{\alpha}
\def\Affb{\beta}
\def\LSig{Split\Sigma}
\def\lyridx{lyr}
\SetAlgoLined
    \KwData{$\defwidth_{total}=\sum_{l=1}^d\width{l}$, $\Sigma\in\{\Sigbar,\SigbarO\}$, $Mode\in\{Numeric,Logical\}$. \\
    $\LSig\subset\bigtimes\limits_{l=1}^d \{0,1\}^{|\Sigma|\times\width{l}}$ is $\Sigma$ split along last into a layer-indexed list of arrays of shape.\\
    network weights $\W{d},\B{d},\ldots,\W{0},\B{0}$\\
    }
\KwResult{If $\Sigma==\Sigbar$ is the set of network states, return a tree whose terminal leaves are affine functions, whose nodes are Min or Max compositions of children, and whose root computation is numerically equivalent to $\net(x)$. If $\Sigma==\Sigbar$ is the set of network states or $\Sigma==\SigbarO$ is the set of network states at the boundary, return a tree whose leaves are affine classifiers, whose nodes are And or Or compositions of children, and whose root is logically equivalent to the statement $\net(x)\geq 0$.}
\KwLet{ \begin{tabularx}{0.8\textwidth}{l l c r}
    OpTree &=ROOT &\% accumulates symbolic rep\\
    LeafInfo.$\mu Index$ &$[1,2,\ldots,|\Sigma|]$ &\\
    LeafInfo.$\tau Index$ &$[1,2,\ldots,|\Sigma|]$&\\
    LeafInfo.$HiddenLayer$ & $\depth$ & \%\text{depth},\\
    LeafInfo.$Affine.(\Affw,\Affb)$ &[$\W{d},\B{d}$] &\%($z\mapsto \Affw z+\Affb{})$\\
    LeafStack & \{ROOT:LeafInfo\}&\%leaf indexed lookup\\
    TerminalLeafs & \{\} & \%empty accumulator
\end{tabularx}}
\BlankLine
\uIf{Mode is Numeric}{
MeetOp,JoinOp=Min,Max\;
}
\uElseIf{Mode is Logicial}{
MeetOp,JoinOp=And,Or\;
}
 \While{LeafStack}{
  LeafSymbol,LeafInfo=LeafStack.pop()\;
  $\lyridx$=LeafInfo.$HiddenLayer$\; 

  RemainPosSig=$\LSig[\lyridx]$[LeafInfo.$\mu Index$]\;
  RemainNegSig=$\LSig[\lyridx]$[LeafInfo.$\tau Index$]\;
  $\Affw,\Affb$=LeafInfo.$Affine$\;
  \%Group remaining neuron states that are same on positive/negative support of $\Affw$\\ 
  JoinSymbols=[]\;
  \For{$\mu linear$,$\mu idx$ in Unique(RemainPosSig$\odot \Affw_+$) }{
    MeetSymbols=[]\;
    \For{$\tau linear$,$\tau idx$ in Unique(RemainNegSig$\odot \Affw_-$) }{
      NewInfo.$Affine$=ComposeAffine([[$\W{\lyridx-1},\B{\lyridx-1}$],[$\mu linear +\tau linear,\Affb$]]\;
      NewInfo.$\mu Index=\mu idx$\ \%tuple of indices from LeafInfo.$\mu Index$\;
      NewInfo.$\tau Index=\tau idx$\;
      NewInfo.$HiddenLayer=lyr-1$\;
      NewSymbol=MakeSymbol(LeafSymbol,$\taul{\lyridx},\mul{\lyridx}$)\;
        \eIf{$lyr-1>0$}{
          LeafStack.update(\{NewSymbol:NewInfo\})\;
          }{
          TerminalLeafs.update(\{NewSymbol:NewInfo\})\;
          }
    MeetSymbols.append(NewSymbol)\;
  }
  JoinSymbols.append( Reduce(MeetSymbols, MeetOP) )\;
  }
  NewBranch=Reduce(JoinSymbols, JoinOP )\;
  OpTree.subs( LeafSymbol, NewBranch )\;
 }
\caption{Network Tree Decomposition}
\label{alg:op_tree}
\end{algorithm}


\clearpage

\clearpage

\subsection{Supporting Theoretical Exposition}\label{sec:proofs} 

This section contains the proofs for the theorems laid out in the paper. All previous results are restated for convenience. Some new results are added to facilitate exposition. 
\PropOpEquiv*
\begin{proof}
There is essentially nothing to prove as both statements are equivalent to $\exists \alpha\forall\beta f(\alpha,\beta)\geq 0$.
\end{proof}

\begin{definition} (Network Operand, $\Op$)\\

For a ReLU DNN composed of weights, $\W{l}$, and biases, $\B{l}$, for $l=0,\ldots,d$, mapping inputs $\x$ to $\bbR$, we define the Network Operand, $\Opl{l}$ (at layer $l$), as follows:
\begin{align*}
    \Opl{0}(\x)&=\B{0}+\W{0}\x\\
    \Recurse{1}(\mul{1},\taul{1},\x)&=\B{1}+\W{1}_{+}\mul{1}(\B{0}+\W{0}\x) - \W{1}_{-}\taul{1}(\B{0}+\W{0}\x)
\end{align*}
Given $\Opl{l}$, define $\Opl{l+1}$ by
\begin{align}
    \Recurse{l+1}(\mul{1},\ldots\mul{l+1},\taul{1},\ldots,\taul{l+1}, \x)&=
    \B{l+1}+\nonumber \\
    &\W{l+1}_{+}\mul{l+1}\Recurse{l}(\mul{1},\ldots\mul{l},\taul{1},\ldots,\taul{l}, \x)-\nonumber\\
    &\W{l+1}_{-}\taul{l+1}\Recurse{l}(\taul{1},\ldots\taul{l},\mul{1},\ldots,\mul{l}, \x)\label{eqn:def_recurse}
\end{align}
taking note that the roles of $\tau$ and $\mu$ are switched in the last term. 

\end{definition}

%


For induction purposes, we define $\netl{l}$ to be the vector valued ReLU network consisting of hidden layers $1,\ldots,l$ of $\net$:
\begin{equation}
\netl{l}(\x)\triangleq\B{l}+\W{l}\relu(\B{l-1}+\W{l-1}\relu(\B{l-2}+\ldots+\W{1}\relu(\B{0}+\W{0}\x)\ldots))\nonumber.
\end{equation}
so that $\netl{d}=\net$ and $\netl{l+1}(\x)=\B{l+1}+\W{l+1}\relu(\netl{l}(\x))$.

\CompositionalTheorem*
\begin{proof}
We first prove Eqn $1$ by induction on $\depth$. Eqn $3$ will follow directly  from $2$ through repeated application of Prop \ref{prop:minmax2bool}. Clearly, $\netl{0}(\x)=\B{0}+\W{0}\x=\Recurse{0}(\x)$, so Eqn $1$ holds for $k=0$ hidden layers. Assume Eqn $1$ is true for $1,\ldots,k$.

By definition of $\netl{k+1}$ and the inductive assumption,
\begin{align*}
    \netl{k+1}(\x)
    &=\B{k+1}+\W{k+1}\relu(\netl{k}(\x))\\
    &=\B{k+1}+\max_{\mul{k+1}}\min_{\taul{k+1}}(\W{k+1}_{+}\mul{k+1}-\W{k+1}_{-}\taul{k+1})(\netl{k}(\x))\\
    &=\B{k+1}+\max_{\mul{k+1}}\min_{\taul{k+1}}(\W{k+1}_{+}\mul{k+1}-\W{k+1}_{-}\taul{k+1})
    (\max_{\mul{k}}\min_{\taul{k}}\cdots \\
    &\qquad\max_{\mul{1}}\min_{\taul{1}} 
    \DotOperand{\mul{1},\ldots\mul{k},\taul{1},\ldots,\taul{k}}{\x}{k} )\\
    &=\B{k+1}+\max_{\mul{k+1}}\min_{\taul{k+1}}\biggl[\W{k+1}_{+}\mul{k+1}(\max_{\mul{k}}\min_{\taul{k}}\cdots \max_{\mul{1}}\min_{\taul{1}}
    \DotOperand{\mul{1},\ldots\mul{k},\taul{1},\ldots,\taul{k}}{\x}{k} )\\
    &\phantom{=}\; -(\max_{\mul{k}}\min_{\taul{k}}\cdots \max_{\mul{1}}\min_{\taul{1}} \W{k+1}_{-}\taul{k+1}
    \DotOperand{\mul{1},\ldots\mul{k},\taul{1},\ldots,\taul{k}}{\x}{k} )\\
    &=\B{k+1}+\max_{\mul{k+1}}\min_{\taul{k+1}}\biggl[\W{k+1}_{+}\mul{k+1}(\max_{\mul{k}}\min_{\taul{k}}\cdots \max_{\mul{1}}\min_{\taul{1}}
    \DotOperand{\mul{1},\ldots\mul{k},\taul{1},\ldots,\taul{k}}{\x}{k} )\\
    &\phantom{=}\;(\min_{\mul{k}}\max_{\taul{k}}\cdots \min_{\mul{1}}\max_{\taul{1}} -\W{k+1}_{-}\taul{k+1}
    \DotOperand{\mul{1},\ldots\mul{k},\taul{1},\ldots,\taul{k}}{\x}{k} )\\
    &=\B{k+1}+\max_{\mul{k+1}}\min_{\taul{k+1}}\biggl[(\max_{\mul{k}}\min_{\taul{k}}\cdots \max_{\mul{1}}\min_{\taul{1}} \W{k+1}_{+}\mul{k+1}
    \DotOperand{\mul{1},\ldots\mul{k},\taul{1},\ldots,\taul{k}}{\x}{k} )\\
    &\phantom{=}\; (\max_{\mul{k}}\min_{\taul{k}}\cdots \max_{\mul{1}}\min_{\taul{1}} -\W{k+1}_{-}\taul{k+1}
    \DotOperand{\taul{1},\ldots\taul{k},\mul{1},\ldots,\mul{k}}{\x}{k} )\biggr]\\
    &=\max_{\mul{k+1}}\min_{\taul{k+1}}\max_{\mul{k}}\min_{\taul{k}}\cdots \max_{\mul{1}}\min_{\taul{1}} \biggl[\B{k+1}+\W{k+1}_{+}\mul{k+1}
    \DotOperand{\mul{1},\ldots\mul{k},\taul{1},\ldots,\taul{k}}{\x}{k} )\\
    &\phantom{=}\; -\W{k+1}_{-}\taul{k+1}\Recurse{k}(\taul{1},\ldots\taul{k},\mul{1},\ldots,\mul{k}, \x)\biggr]\\
    &=\max_{\mul{k+1}}\min_{\taul{k+1}}\max_{\mul{k}}\min_{\taul{k}}\cdots \max_{\mul{1}}\min_{\taul{1}}\biggl[
    \DotOperand{\mul{1},\ldots\mul{k},\mul{k+1},\taul{1},\ldots,\taul{k},\taul{k+1}}{ \x}{k+1}\biggr]
\end{align*}

Thus by induction Eqn $1$ holds for any depth. Now we can apply Prop \ref{prop:minmax2bool} recursively to derive Eqn $2$.

\begin{align*}
\biggl[\net(\x)\geq0\biggr]
&\Longleftrightarrow\biggl[\max_{\mul{d}}\min_{\taul{d}} \max_{\mul{d-1}}\min_{\taul{d-1}}\cdots \max_{\mul{1}}\min_{\taul{1}} \Recurse{d}(\mul{1},\ldots\mul{d},\taul{1},\ldots,\taul{d}, x)\geq 0\biggr]\\
&\Longleftrightarrow \bigvee_{\mul{d}}\bigwedge_{\taul{d}}\biggl[ \max_{\mul{d-1}}\min_{\taul{d-1}}\cdots \max_{\mul{1}}\min_{\taul{1}} \Recurse{d}(\mul{1},\ldots\mul{d},\taul{1},\ldots,\taul{d}, x)\geq 0\biggr]\\
&\Longleftrightarrow \bigvee_{\mul{d}}\bigwedge_{\taul{d}} \bigvee_{\mul{d-1}}\bigwedge_{\taul{d-1}}\biggl[\cdots \max_{\mul{1}}\min_{\taul{1}}\Recurse{d}(\mul{1},\ldots\mul{d},\taul{1},\ldots,\taul{d}, x)\geq 0 \biggr]\\
&\vdots\\
&\Longleftrightarrow \bigvee_{\mul{d}}\bigwedge_{\taul{d}} \bigvee_{\mul{d-1}}\bigwedge_{\taul{d-1}}\cdots \bigvee_{\mul{1}}\bigwedge_{\taul{1}}\biggl[
\Recurse{d}(\mul{1},\ldots\mul{d},\taul{1},\ldots,\taul{d}, x)\geq 0\biggr]
\end{align*}
\end{proof}

For the following, we recall the following notation:
For $J\subset[\depth]$, $\Sigbarl{J}$[$\SigbarOl{J}$] 
is the projection of $\Sigbar$[$\SigbarO$] onto the coordinates indexed by $J$. 
The symbols 
$\mubar^{[l]}=(\mul{1},\ldots,\mul{l})$ for $l\leq d$ 
and 
$\mubar=\mubar^{[d]}=(\mul{1},\ldots,\mul{d})$
to be equivalent (representing the same quantity) in any context they appear together. 
For example, if $\mubar\in\Sigbar$ then $\mubar^{[l]}\in\Sigbarl{[l]}$. 
We will here consider $\mubar^{[l]}$ as a concatenated vector in $\bbR^{\sum_{i=1}^l\width{i}}$.

\def\fixmu{\hat{\mu}}
\def\fixtau{\hat{\tau}}
\begin{lemma}The Fundamental Lemma of the Net Operand\\
\label{lem:fund_operand} Let $\net$ be a ReLU network with net operand, $\Op$, and $\sigbar:\bbR^{\width{0}}\mapsto\Sigbar$, mapping inputs, $\x$, to network states, $\sigbar(x)$. Then for arbitrary binary vectors $\fixmu,\fixtau$ the following relations hold
\begin{align}
     \min_{\taubar\in\Sigbar}\DotOp{\sigbar(\x),\taubar}{\x} = &\net(\x) =\max_{\mubar\in\Sigbar}\DotOp{\mubar,\sigbar(\x)}{\x}
     \label{eqn:fund_operand}\\
    \min_{\taubar\in\Sigbar}\DotOp{\fixmu{},\taubar}{\x}\leq\DotOp{\fixmu,\sigbar(\x)}{\x}
    \leq &\net(\x) 
    \leq \DotOp{\sigbar(\x),\fixtau}{\x} 
    \leq\max_{\mubar\in\Sigbar}
    \DotOp{\mubar,\fixtau}{\x}.
    \label{eqn:net_operand}
\end{align}
\end{lemma}

It's worth pointing out that Lemma \ref{lem:fund_operand} implies $\net(x)=\DotOp{\sigbar(\x),\sigbar(\x)}{\x}$, as was claimed in the text. This is in fact how numerical equality is achieved in Equation 1.

\begin{proof}
We claim it is sufficient to show
\begin{align}
     \max_{\mubar\in\Sigbar}\DotOp{\mubar,\sigbar(\x)}{\x}\leq \net(x) 
     \leq \min_{\taubar\in\Sigbar}\DotOp{\sigbar(\x),\taubar}{\x}.
     \label{eqn:lem_fund_reduced}
\end{align}
To see this, note the following chain of inequalities. They make use of the fact that the maximum[minimum] is always greater[less] than or equal to any particular fixed value.
\begin{align*}
    &\min_{\taubar\in\Sigbar}\DotOp{\fixmu{},\taubar}{\x}
    \leq\DotOp{\fixmu,\sigbar(\x)}{\x}
    \leq \max_{\mubar\in\Sigbar}\DotOp{\mubar,\sigbar(\x)}{\x}\leq \net(x)\\
    &\leq\min_{\taubar\in\Sigbar}\DotOp{\sigbar(\x),\taubar}{\x} 
    \leq \DotOp{\sigbar(\x),\fixtau}{\x} 
    \leq\max_{\mubar\in\Sigbar}
    \DotOp{\mubar,\fixtau}{\x}.
\end{align*}
Thus the we obtain the inequalities in Eqn 6. Analyzing the special case that $\fixmu{}=\fixtau{}=\sigbar(x)$, we also obtain 
\begin{align*}
    \min_{\taubar\in\Sigbar}\DotOp{\sigbar(\x),\taubar}{\x}
    \leq \max_{\mubar\in\Sigbar}\DotOp{\mubar,\sigbar(\x)}{\x}
    \leq \min_{\taubar\in\Sigbar}\DotOp{\sigbar(\x),\taubar}{\x}
    \leq  \max_{\mubar\in\Sigbar}
    \DotOp{\mubar,\sigbar(\x)}{\x}.
\end{align*}
Thus we obtain the equalities in Eqn. 5. 
Now we turn to proving Equation 7.

Consider the term-wise expansion of $\Op$. We claim every term containing $\mul{1}$ has a leading $+$ and every term containing $\taul{1}$ has a leading $-$. When $d=1$, this is obvious. And, if it is true in the expansion of $\DotOperand{\mu^{[l]},\tau^{[l]}}{\x}{l}$, then by definition (Eqn. $4$) it is true in the expansion of $\DotOperand{\mu^{[l+1]},\tau^{[l+1]}}{\x}{l+1}$, since the minus sign in front of $\DotOperand{\tau^{[l]},\mu^{[l]}}{\x}{l}$ also features $\mul{1}$ and $\taul{1}$ switching roles.

\def\posveci{v^a_+}
\def\posvecii{v^b_+}
Since every matrix in the expansion of $\Op$ is entry-wise nonnegative, for every fixed
$\mul{2},\ldots,\mul{d}$ and 
$\taul{2},\ldots,\taul{d}$,
we can by combining terms obtain some bias $\beta$ and nonnegative vectors $\posveci,\posvecii\in\bbR^{\width{1}}_+$ so that 
\begin{align*}
    \DotOp{\mubar,\taubar}{\x}&=\beta+\posveci\mul{1}(\B{0}+\W{0}\x)-\posvecii\taul{1}(\B{0}+\W{0}\x)\\
    &=\beta+\posveci\mul{1}\netl{0}(\x)-\posvecii\taul{1}\netl{0}(\x)\\
\end{align*}
Therefore, for every fixed 
$\mul{2},\ldots,\mul{d}$ and 
$\taul{2},\ldots,\taul{d}$,
we may consider $\taul{1}=\sigl{1}(\x)$ to be an optimal choice (for any $\mul{1}$) and $\mul{1}=\sigl{1}(\x)$ to be an optimal choice (for any $\taul{1}$). We have shown
that for any $\mul{2},\ldots,\mul{d}$ and 
$\taul{2},\ldots,\taul{d}$, that
\begin{align*}
    &\phantom{\leq}\max_{\mul{1}}\DotOp{\mul{1},\mul{2},\ldots,\mul{d},\sigl{1}(\x),\taul{2},\ldots,\taul{d}}{\x}\\
    &\leq \DotOp{\sigl{1}(\x),\mul{2},\ldots,\mul{d},\sigl{1}(\x),\taul{2},\ldots,\taul{d}}{\x}\\
    &\leq\min_{\taul{1}}\DotOp{\sigl{1}(\x),\mul{2},\ldots,\mul{d},\taul{1},\taul{2},\ldots,\taul{d}}{\x}
\end{align*}

Now suppose for some $k\leq d$ that 
$\mul{1}=\taul{1}=\sigl{1}(\x),\ldots, \mul{k-1}=\taul{k-1}=\sigl{k-1}(\x)$. Let 
$\mul{k+1},\ldots,\mul{d}$ and 
$\taul{k+1},\ldots,\taul{d}$ be arbitrary and fixed. 
It is quite clear by substitution that $\DotOperand{\sigbar^{[k-1]},\sigbar{[k-1]}}{\x}{k-1}=\netl{k-1}(\x)$. Then we may substitute $\netl{k-1}(\x)$ for every $\Opl{k-1}$ in the expansion of $\Op$. We are left with the same expansion as before but with $(\mul{k},\taul{k})$ in place of $(\mul{1},\taul{1})$, $d-k+1$ in place of $d$, and $\netl{k-1}(\x)$ in place of $\netl{0}(\x)$. Accordingly, we can conclude by the same logic that $\taul{k}=\sigl{k}(\x)$ minimizes $\Op$, and that $\mul{k}=\sigl{k}(\x)$ maximizes $\Op$, as soon as $\mul{l}=\taul{l}=\sigl{l}(\x)$ for $l<k$.

If we apply this reasoning recursively, we can see
\begin{align*}
    \max_{\mubar}\DotOp{\mubar,\sigbar(\x)}{\x}
    &=\max_{\mul{d},\ldots,\mul{1}}\DotOp{\mul{1},\ldots\mul{d},\sigbar(\x)}{\x}\\
    &\leq\max_{\mul{d},\ldots,\mul{2}}\DotOp{\sigl{1}(\x),\mul{2},\ldots\mul{d},\sigbar(\x)}{\x}\\
    &\leq\max_{\mul{d},\ldots,\mul{3}}\DotOp{\sigl{1}(\x),\sigl{2}(\x),\mul{3},\ldots\mul{d},\sigbar(\x)}{\x}\\
    &\vdotswithin{\leq}\\
    &\leq \DotOp{\sigbar(\x),\sigbar(\x)}{\x}=\net(\x)\\
    &\vdotswithin{\leq}\\
    &\leq\min_{\taul{d},\ldots,\taul{3}}\DotOp{\sigbar(\x),\sigl{1}(\x),\sigl{2}(\x),\taul{3},\ldots\taul{d}}{\x}\\
    &\leq\min_{\taul{d},\ldots,\taul{2}}\DotOp{\sigbar(\x),\sigl{1}(\x),\taul{2},\ldots\taul{d}}{\x}\\
    &\leq\min_{\taul{d},\ldots,\taul{1}}\DotOp{\sigbar(\x),\taul{1},\ldots\taul{d}}{\x}\\
    &=\min_{\mubar}\DotOp{\mubar,\sigbar(\x)}{\x}
\end{align*}
Since we have shown Equation $7$, this concludes the proof.

\end{proof}

This Theorem is not essential to the main story, but the last line (Eqn. 10) is referenced as "Mode=Numeric" in Algorithm \ref{alg:op_tree}. It can be thought of as a numeric analogue of Theorem \ref{thm:sig_bdry} that models the \textit{function} $\net(\x)$ using Min and Max. Of course, this model requires that all network states, $\Sigbar$, participate. Not just those at the boundary, $\SigbarO$.

\begin{restatable}{theorem}{SwapMinMaxTheorem}\label{thm:swap_minmax}
The order of the operands in Eqn 1 may be switched as         
    \begin{align}
        \net(\x)&=\max_{\mul{d}}\min_{\taul{d}}\cdots\max_{\mul{1}}\min_{\taul{1}}
        \DotOp{\mubar,\taubar}{\x}\nonumber\\
        &=\max_{\mubar\in\Sigbar}\min_{\taubar\in\Sigbar}       
        \DotOp{\mubar,\taubar}{\x}\label{eqn:4-Mm}\\
        &=\min_{\taubar\in\Sigbar}\max_{\mubar\in\Sigbar}
        \DotOp{\mubar,\taubar}{\x}\label{eqn:4-mM}\\
        &=\max_{\mul{d}\in\Sigbarl{d}}\min_{\taul{d}\in\Sigbarl{d}}
        \max_{\{\mul{d-1}|(\mul{d-1},\mul{d})\in\Sigbarl{d-1,d}\}}
        \min_{\{\taul{d-1}|(\taul{d-1},\taul{d})\in\Sigbarl{d-1,d} \}}\cdots\nonumber\\
        &\max_{\{\mul{1}|(\mul{1},\ldots,\mul{d})\in\Sigbar\}}
        \min_{\{\taul{1}|(\taul{1},\ldots,\taul{d})\in\Sigbar\}}
        \DotOp{\mubar,\taubar}{\x}       \label{eqn:EquivSig}
    \end{align}
\end{restatable}

\begin{proof}
To show Equations 8 and 9, simply use the equality relations in Lemma \ref{lem:fund_operand} and the minmax inequality:
\begin{align*}
    \net(\x)&=\min_{\taubar\in\Sigbar}\DotOp{\sigbar(\x),\taubar}{\x} \\
    &\leq \max_{\mubar\in\Sigbar}\min_{\taubar\in\Sigbar}\DotOp{\mubar,\taubar}{\x}
    \leq \min_{\taubar\in\Sigbar}\max_{\mubar\in\Sigbar}\DotOp{\mubar,\taubar}{\x}\\
    &\leq \max_{\mubar\in\Sigbar}\DotOp{\mubar,\sigbar(\x)}{\x}\\
    &=\net(\x)
\end{align*}
\end{proof}
The proof of Eqn 10 could have been incorporated into the minmax inequality step above, but instead we treat it separately for notational clarity. In the next chain of equations we suppress the index set notation for a few lines so as not to obscure the shuffling of Min and Max operators.

\begin{align}
    \net(\x)&=
    \min_{\taubar}\max_{\mubar}
    \DotOp{\mubar,\taubar}{\x}\label{eqn:start_mM}\\
    &=\min_{\taul{d}\in\Sigbarl{d}}  \min_{\{\taul{d-1}|(\taul{d-1},\taul{d})\in\Sigbarl{d-1,d}\}} \min_{\{\taul{1}|(\taul{1},\ldots,\taul{d})\in\Sigbar\}}\cdots\nonumber\\
    &\qquad\max_{\mul{d}\in\Sigbarl{d}}  \max_{\{\mul{d-1}|(\mul{d-1},\mul{d})\in\Sigbarl{d-1,d}\}} \max_{\{\mul{1}|(\mul{1},\ldots,\mul{d})\in\Sigbar\}}
    \DotOp{\mubar,\taubar}{\x}\nonumber\\
    &=\min_{\taul{d}}\cdots\min_{\taul{1}}\max_{\mul{d}}\cdots\max_{\mul{1}}
    \DotOp{\mubar,\taubar}{\x}\nonumber\\
    &\geq\min_{\taul{d}}\cdots\min_{\taul{2}}\max_{\mul{d}}
    \cdots\max_{\mul{1}}\min_{\taul{1}}\DotOp{\mubar,\taubar}{\x}\nonumber\\
    &\geq\min_{\taul{d}}\cdots\min_{\taul{3}}\max_{\mul{d}}
    \cdots\max_{\mul{2}}\min_{\taul{2}}\max_{\mul{1}}\min_{\taul{1}}
    \DotOp{\mubar,\taubar}{\x}\nonumber\\
    &\vdots\nonumber\\
    &\geq\max_{\mul{d}\in\Sigbarl{d}}\min_{\taul{d}\in\Sigbarl{d}}\cdots
    \max_{\{\mul{1}|(\mul{1},\ldots,\mul{d})\in\Sigbar\}}
    \min_{\{\taul{1}|(\taul{1},\ldots,\taul{d})\in\Sigbar\}}
    \DotOp{\mubar,\taubar}{\x}=\text{Eqn 10}\nonumber\\
    &\geq\nonumber\\
    &\vdots\nonumber\\
    &\geq \max_{\mubar}\min_{\taubar}\DotOp{\mubar,\taubar}{\x}\label{eqn:end_Mm}\\
    &=\net(\x)\nonumber
\end{align}

    %

\begin{proposition}\label{prop:SigPM} Let $\net$ be a ReLU network with net operand, $\Op$, and states $\Sigbar=\SigbarP\cup\SigbarM$. Then
\begin{equation}
    \biggl[\net(\x)\geq 0\biggr]
    \Leftrightarrow\biggl[\max_{\mubar\in\SigbarP}\min_{\taubar\in\SigbarM}\DotOp{\mubar,\taubar}{\x}\geq 0\biggr]
    \Leftrightarrow \biggl[\min_{\taubar\in\SigbarM}\max_{\mubar\in\SigbarP}\DotOp{\mubar,\taubar}{\x}\geq 0\biggr] \label{eqn:SigPM}
\end{equation}
\begin{proof}
If $\net(\x)\geq 0$, then $\sigbar(x)\in\SigbarP$. Therefore
\begin{align*}
\max_{\mubar\in\SigbarP}\min_{\taubar\in\SigbarM}\DotOp{\mubar,\taubar}{\x}\geq \min_{\taubar\in\Sigbar}\DotOp{\sigbar(\x),\taubar}{\x}=\net(\x)\geq 0.
\end{align*}
Conversely, if $\net(\x)<0$, then $\sigbar(x)\in\SigbarM$ and 
\begin{align*}
\min_{\taubar\in\SigbarM}\max_{\mubar\in\SigbarP}\DotOp{\mubar,\taubar}{\x}\leq \max_{\mubar\in\Sigbar}\DotOp{\mubar,\sigbar(\x)}{\x}=\net(\x)< 0.
\end{align*}
We have proved the first and third arrows of the below sequence of implications,
\begin{equation}
    \biggl[\net(\x)\geq 0\biggr]\Rightarrow
    \biggl[\max_{\mubar\in\SigbarP}\min_{\taubar\in\SigbarM}\DotOp{\mubar,\taubar}{\x}\geq 0\biggr]\Rightarrow
    \biggl[\min_{\taubar\in\SigbarM}\max_{\mubar\in\SigbarP}\DotOp{\mubar,\taubar}{\x}\geq 0\biggr]\Rightarrow
    \biggl[\net(\x)\geq 0\biggr].\nonumber
\end{equation}
The middle one is a consequence of the minmax inequality.
\end{proof}
\end{proposition}

We are almost done with our theoretical development, and we have not yet used the fact that $\Op$ is a linear (affine) function of $x$. That changes with the next Theorem, which requires Farkas' Lemma. To linearize, we define $\xv=\xcol\in\bbR^{\width{0}+1}$ to be the embedding of our inputs in the affine plane. 

\BoundaryTheorem*

We will reuse the same trick operator shuffling technique from the proof of Theorem 5 (From Eqn. 11 to Eqn. 12). If we can prove equivalence of the form in Equation 13, but with $(\SigbarO,\SigbarO)$ in place of $(\SigbarP,\SigbarM)$ for the index sets, then we can use $2d$ applications of the minmax inequality to sandwich the actual term we care about. The reader should refer to these manipulations, as it is a very handy trick!

\def\muplus{\mubar_{+}}
\def\tauminus{\taubar_{-}}
\def\muzero{\mubar_{(0)}}
\def\vv{v_{\mubar,x}}
\def\vvp{v_{\muplus,x}}
\newcommand{\DotV}[2]{v_{#1,#2}}
\newcommand{\DotVp}[1]{v_{\muplus{},#1}}

\begin{proof}

Let $X_{\taubar}=\{x|\sigbar(x)=\taubar\}$.
Each pair of binary vectors, $\mubar,\taubar$, corresponds to a vector $\DotV{\mubar}{\taubar}\in \bbR^{\width{0}+1}$ so that $\DotOp{\mubar,\taubar}{\x}=\DotV{\mubar}{\taubar}\cdot\xv$. 
Consider $\muplus{}\in\SigbarP\setminus\SigbarM$.
For $\tauminus\in\SigbarM$, let $K^{\muplus}_{\tauminus}$ be the (convex) region of the input where $\DotVp{\tauminus}\cdot\xv=\min_{\taubar\in\SigbarM}\DotVp{\taubar}\cdot\xv$. 

Suppose that $\net(\x)\geq 0$. Let $x$ be an arbitrary input. It must, for some $\tauminus$, be in one of the convex sets, say $K^{\muplus}_{\tauminus}$. Now suppose further that for this binary vector, $\tauminus$, we have $\DotV{\muplus{}}{\tauminus{}}\cdot\xv \geq \DotV{\mubar}{\tauminus{}}\cdot\xv$
 for all $\mubar\in\SigbarP$. Suppose that for $\forall\mubar$. In this case we have 
\begin{align*}
\DotV{\muplus}{\tauminus}&\geq \max_{\mubar\in\SigbarP}\DotV{\mubar}{\tauminus}\geq 0\\
&\geq \max_{\mubar\in\SigbarP}\min_{\taubar\in\SigbarM}\DotV{\mubar}{\taubar}\geq 0
\end{align*} 
To rephrase, let $A$ be a matrix with the vectors $\{\DotV{\mubar}{\tauminus}\}_{\mubar\in\SigbarP}$ for columns. We showed $\not \exists\xv\in K^{\muplus}_{\tauminus}$ such that $A^T\xv\succeq 0 $ and $\DotV{\muplus{},\tauminus{}}\cdot\xv<0$. Therefore, Farkas' Lemma tells us that we can find positive nonnegative scalars, $\{\alpha_{\mubar}|\mubar\in\SigbarP\}$, and some dual vector $\eta\in (K^{\muplus}_{\tauminus})^*$
such that
\begin{align}
&\DotV{\muplus}{\tauminus}=\eta+\sum_{\mubar\in\SigbarP}\alpha_{\mubar}(\DotV{\muplus,\tauminus{}}-\DotV{\mubar}{\tauminus{}})\nonumber\\
\Leftrightarrow &\sum_{\mubar\in\SigbarP}\alpha_{\mubar}v_{\mubar,\tauminus}+
\biggl(1-\biggl(\sum_{\mubar\in\SigbarP}\alpha_{\mubar}\biggr)\biggr)(v_{\muplus{},\tauminus})=\eta
\label{alpha_lt1}\\
\Leftrightarrow &\sum_{\mubar\in\SigbarP}\alpha_{\mubar}v_{\mubar,\tauminus}=\eta+
\biggl(\biggl(\sum_{\mubar\in\SigbarP}\alpha_{\mubar}\biggr)-1\biggr)(v_{\muplus{},\tauminus})\label{alpha_gt1}
\end{align}

Here, we point out that $\sum_{\mubar\in\SigbarP}\alpha_{\mubar}<1$ is not possible. Since the RHS of Equation 14 has nonnegative dot product with every $\xv\in K^{\muplus}_{\tauminus}$, we know that at least one of the vectors on the LHS must as well. This would imply nonnegativity of maximum over all of $\SigbarP$, given by, $\max_{\mubar\in\SigbarP} v_{\mubar,\tauminus}\xv\geq 0$ $\forall\xv\in K^{\muplus}_{\tauminus}$. But, we know this to be false, since $\tauminus\in\SigbarM$, which guarantees the existence of at least some $\xv$ with $\max_{\mubar}\DotV{\mubar}{\tauminus{}}\cdot\xv<0$.
But now, consider the RHS of Eqn 15, which is nonnegative. We see that every time $v_{\muplus,\tauminus}\xv\geq 0$, then also we are guaranteed some term in the left summation, say indexed by $\mu'\in\SigbarP$, such that $v_{\mu',\tau^*}\bar{y}\geq 0$ (we may assume not all $\alpha_{\mubar}=0$). In other words, we have the first arrow of
\begin{align*}
\biggl[\net(\x)\geq 0\biggr]&\Rightarrow\biggl[\max_{\mubar\in\SigbarO}\min_{\taubar\in\SigbarM}\DotOp{\mubar,\taubar}{\x}\geq 0\biggr]
\Rightarrow\biggl[\max_{\mubar\in\SigbarO}\min_{\taubar\in\SigbarO}\DotOp{\mubar,\taubar}{\x}\geq 0 \biggr]\Rightarrow \text{Eqn.\ref{eqn:EquivSigO}}\\
&\Rightarrow\biggl[\min_{\taubar\in\SigbarO}\max_{\mubar\in\SigbarO}\DotOp{\mubar,\taubar}{\x}\geq 0 \biggr]
\Rightarrow\biggl[\min_{\taubar\in\SigbarO}\max_{\mubar\in\SigbarP}\DotOp{\mubar,\taubar}{\x}\geq 0 \biggr]\Rightarrow\biggl[\net(\x)\geq 0\biggr]
\end{align*}
The contrapositive of the very last arrow (starting with $\net(x)<0$) is extremely similar to the first, so we omit it. The remaining arrows are all either minmax inequality or obtained shrinking or growing the set being optimized over. Thus completing the proof.
\end{proof}



\end{document}